\documentclass[twoside]{article}
\usepackage[accepted]{aistats2020new}

\usepackage{url}            % simple URL typesetting
\usepackage{booktabs}       % professional-quality tables
\usepackage{amsfonts}       % blackboard math symbols
\usepackage{nicefrac}       % compact symbols for 1/2, etc.
\usepackage{microtype}      % microtypography
\usepackage{url}
\usepackage{graphicx}
\usepackage[font=small]{caption}
\usepackage{amsmath}

\usepackage{amsfonts}
\usepackage{subcaption}
\usepackage{bm}
\usepackage{graphicx} % omit 'demo' for real document
\usepackage[round,sort]{natbib}

\usepackage{xcolor}
\usepackage{algorithm,algorithmic}
\usepackage{tabularx} % also loads 'array' package
\usepackage{amsthm}
\usepackage{amsfonts}
\usepackage{amsmath}
\usepackage{amssymb}
\usepackage{mathrsfs}
\usepackage{multirow}
\usepackage{bbm}
\usepackage{wrapfig}

\newtheorem{theorem}{Theorem}

\newtheorem{corollary}{Corollary}

\newtheorem{proposition}{Proposition}
\newtheorem{lemma}{Lemma}

\DeclareMathOperator*{\argmax}{arg\,max}
\DeclareMathOperator*{\argmin}{arg\,min}

\usepackage{enumitem}
\usepackage{wrapfig}

\newcommand{\st}{\text{ s.t. }}
\newcommand{\beq}{\begin{equation}}

\newcommand{\eeq}{\end{equation}}
\newcommand{\beqs}{\begin{equation*}}
\newcommand{\eeqs}{\end{equation*}}

\renewcommand{\L}{\mathcal{L}}
\newcommand{\LL}{\mathbb{L}}
\newcommand{\R}{\mathbb{R}}
\newcommand{\EE}{\mathcal{E}}

\newcommand{\U}{\mathcal{U}}
\newcommand{\G}{\mathcal{G}}
\newcommand{\M}{\mathcal{M}}

\newcommand{\define}{\stackrel{\text{def.}}{=}}

\newcommand{\E}{\mathbb{E}}
\newcommand{\X}{\mathcal{X}}
\newcommand{\C}{\mathcal{C}}
\newcommand{\D}{\mathcal{D}}

\renewcommand{\Gamma}{\bm \mu}

\newcommand{\1}{\mathbbm{1}}
\newcommand{\<}{\langle}
\renewcommand{\>}{\rangle}

\renewcommand{\tilde}{\widetilde}

\newcommand{\ceil}[1]{\lceil #1 \rceil}

\newcommand{\V}{\mathcal{V}}
\newcommand{\T}{\mathsf{T}}

\renewcommand{\P}{\mathcal{P}}
\renewcommand{\T}{\mathcal{T}}

\usepackage[framemethod=TikZ]{mdframed}
\usepackage{amsthm}
\newcounter{theo}[section] 
\renewcommand{\thetheo}{\arabic{section}.\arabic{theo}}

\newcounter{lem}[section] 
\renewcommand{\thelem}{\arabic{section}.\arabic{lem}}

\newcounter{prf}[section]
\renewcommand{\theprf}{\arabic{section}.\arabic{prf}}

\newcounter{cor}[section]
\renewcommand{\thecor}{\arabic{section}.\arabic{cor}}

\newcounter{prop}[section] 
\renewcommand{\thelem}{\arabic{section}.\arabic{prop}}

\begin{document}

\runningtitle{Convergence Rates of Smooth Message Passing with Rounding}

\twocolumn[

\aistatstitle{Convergence Rates of Smooth Message Passing with Rounding in Entropy-Regularized MAP Inference}

\aistatsauthor{Jonathan N. Lee$^*$\And Aldo Pacchiano$^*$ \And Michael I. Jordan  }
\aistatsaddress{Stanford University\And UC Berkeley\And UC Berkeley} ]

\begin{abstract}
Maximum a posteriori (MAP) inference is a fundamental computational paradigm for statistical inference. In the setting of graphical models, MAP inference entails solving a combinatorial optimization problem to find the most likely configuration of the discrete-valued model. Linear programming (LP) relaxations in the Sherali-Adams hierarchy are widely used to attempt to solve this problem, and smooth message passing algorithms have been proposed to solve regularized versions of these LPs with great success. This paper leverages recent work in entropy-regularized LPs to analyze convergence rates of a class of edge-based smooth message passing algorithms to $\epsilon$-optimality in the relaxation. With an appropriately chosen regularization constant, we present a theoretical guarantee on the number of iterations sufficient to recover the true integral MAP solution when the LP is tight and the solution is unique.

\end{abstract}

\section{INTRODUCTION}

Undirected graphical models are a central modeling formalism in machine learning, providing a compact and powerful way to model dependencies between variables. Here we focus on the important class of discrete-valued pairwise models. Inference in discrete-valued graphical models has applications in many areas including computer vision, statistical physics, information theory, and genome research~\citep{antonucci2014probabilistic,wainwright2008graphical,mezard2009information}.  

We focus on the problem of identifying a configuration of all variables that has highest probability, termed \emph{maximum a posteriori} (MAP) inference. This problem has an extensive literature across multiple communities, where it is described by various names, including energy minimization \citep{kappes2013comparative} and constraint satisfaction \citep{schiex1995valued}.  In the binary case, the MAP problem is sometimes described as quadratic-pseudo Boolean optimization \citep{hammer1984roof} and it is known to be NP-hard to compute exactly~\citep{kolmogorov2004energy, cooper1990computational} or even to approximate \citep{dagum1993approximating}. Consequently, much work has attempted to identify settings where polynomial-time methods are feasible. We call such settings ``tractable'' and the methods ``efficient.''
A general framework for obtaining tractable methodology involves ``relaxation''---the MAP problem is formulated as an integer linear program (ILP) and is then relaxed to a linear program (LP). If the vertex at which the LP achieves optimality is integral, then it provides an exact solution to the original problem. In this case we say that the LP is tight. If the LP is performed over the convex hull of all integral assignments, otherwise known as the marginal polytope $\mathcal{M}$, then it will always be tight. Inference over the marginal polytope is generally intractable because it requires exponentially many constraints to enforce global consistency.

A popular workaround is to relax the marginal polytope to the local polytope $\LL_2$~\citep{wainwright2008graphical}. Instead of enforcing global consistency, the local polytope enforces consistency only over pairs of variables, thus yielding pseudo-marginals which are pairwise consistent but may not correspond to any true global distribution.  The number of constraints needed to specify the local polytope is linear in the number of edges.  More generally, \cite{sherali1990hierarchy} introduced a series of successively tighter relaxations of the marginal polytope, or convex hull, while retaining control on the number of constraints. However, even with these relaxations, it has been observed that standard LP solvers do not scale well \citep{yanover2006linear}, motivating the study of solvers that exploit the structure of the problem, such as message passing algorithms.

Of particular interest to this paper are \textit{smooth} message passing algorithms, i.e. algorithms derived from regularized versions of the relaxed LP \citep{meshi2012convergence, savchynskyy2011study,savchynskyy2012efficient, hazan2008convergent, ravikumar2010message}. These regularized LPs conduce to efficient optimization in practice and have the special property that their fixed points are unique and optimal; however, this comes at the cost of solving an approximation of the true MAP problem and, without rounding, they do not recover integral solutions in general. Non-asymptotic convergence rates to the optimal \textit{regularized} function value have been studied \citep{meshi2012convergence}, but guarantees on the number of iterations sufficient to recover the optimal integral assignment of the true MAP problem have not been considered to our knowledge. 

In this work we provide a sharp analysis of the entropy-regularized MAP inference problem with Sherali-Adams relaxations. We first characterize the approximation error of the regularized LP in $l_1$ distance, based on new results on entropy-regularized LPs \citep{weed2018explicit}. 
We then analyze an edge-based smooth message passing algorithm, modified from the algorithms described in \citet{werner2007linear} and \citet{ravikumar2010message}. We prove a $O(1/\epsilon^2)$ rate of convergence of iterates in $l_1$ distance. Combining the approximation error and convergence results, we present a guarantee on the number of iterations sufficient to recover of the true integral MAP assignment using a standard vertex rounding scheme when the LP relaxation is tight and the solution is unique.

\section{RELATED WORK}\label{section::related work}

The idea of entropy regularization to aid optimization in inference problems is well studied. It is well known that solving a scaled and entropy-regularized linear program over the marginal polytope yields the scaled Gibbs free energy, intimately related to the log partition function, when the temperature parameter equals one \citep{wainwright2008graphical}.  As the temperature parameter is driven to zero, the calculation of the free energy reduces to the value of the MAP problem. However, this problem is intractable due to the difficulty of both computing the exact entropy and characterizing the marginal polytope \citep{deza2009geometry}. 
Therefore, there has been much work in trying to turn this observation into tractable inference algorithms. The standard Bethe approximation instead minimizes an approximation of the true entropy \citep{bethe1935statistical}. It was show by \cite{yedidia2003understanding} that fixed points of the loopy belief propagation correspond to its stationary points, but still the optimization problem resulting from this approximation is non-convex and convergence is not always guaranteed.

To alleviate convergence issues, much work has considered convexifying the free energy problem leading to classes of convergent convex belief propagation often derived directly from convex regularizers \citep{meshi2009convexifying,heskes2006convexity,hazan2008convergent,johnson2008convex,savchynskyy2012efficient}.  For instance, \cite{weiss2012map} proposed a general convexified belief propagation and explored some sufficient conditions that enable heuristically recovering the MAP solution of the LP via a convex sum-product variant.
However, the approximation error was still unclear and non-asymptotic convergence rates were not considered. A number of algorithms have also been proposed to directly optimize the unregularized LP relaxation often with only asymptotic convergence guarantees such as block-coordinate methods \citep{werner2007linear,globerson2008fixing,kovalevsky1975diffusion,tourani2018mplp,kappes2013comparative} and tree-reweighted message passing \citep{wainwright2005map,kolmogorov2006convergent}. 
The relationship between the regularized and unregularized problems can equivalently be viewed as applying a soft-max to the dual objective typically considered in the latter to recover that of the former \citep{nesterov2005smooth,sontag2011introduction}. 
Many other convergent methods exist such as augmented Lagrangian \citep{martins2011augmented, meshi2011alternating}, bundle \citep{kappes2012bundle}, and steepest descent  \citep{schwing2012globally, schwing2014globally} approaches, but again they are difficult to compare without rates.

Most closely related to our work is recent work in convergence analysis of certain smoothed message passing algorithms that aim to solve the regularized LP objective. \cite{savchynskyy2011study} proposed an accelerated gradient method that achieves $O(1/\epsilon)$ convergence to the optimal regularized dual objective value. Convergence of the primal iterates was only shown asymptotically. \cite{meshi2012convergence} considered a general dual coordinate minimization algorithm based on the entropy-regularized MAP objective. They proved upper bounds on the rate of convergence to the optimal regularized dual objective \textit{value}; however, closeness to the true MAP assignment was not formally characterized. Furthermore, convergence in the dual objective value again does not make it easy to determine when the true MAP assignment can be recovered. \cite{meshi2015smooth} later studied the benefits of adding a quadratic term to the LP objective instead and proved similar guarantees. \cite{ravikumar2010message} also considered entropic and quadratic regularization, using a proximal minimization scheme with inner and outer loops. They additionally provided rounding guarantees to recover true primal solutions. However, as noted by the authors, the inexact calculation of the inner loop prevents a convergence rate analysis once combined with the outer loop. Additionally, rates on the inner loop convergence were not addressed.

The approach of this paper can be understood as the bridging the gap between \cite{meshi2012convergence} and \cite{ravikumar2010message}. 
Our first contribution is a characterization of the approximation error of the entropy-regularized MAP inference problem. We then study an edge-based message passing algorithm that solves the regularized LP, 
% which is essentially the inner loop of the proximal updates of \cite{ravikumar2010message}. 
which is essentially a smoothed max-sum diffusion \citep{werner2007linear} or the inner loop of the proximal steps of \cite{ravikumar2010message}.
For our main contribution, we provide non-asymptotic guarantees to the integral MAP assignment for this message passing algorithm when the LP is tight and the solution is unique. To our knowledge, this is the first analysis with rates guaranteeing recovery of the true MAP assignment for smooth methods.

\section{BACKGROUND}

	 We denote the $d$-dimensional probability simplex as $	\Sigma_d \define \left\{ p \in \R_+^d \: : \: \sum_{i} p_i = 1 \right\}$.
	The set of joint distributions which give rise to $p, q \in \Sigma_d$ is defined as $
	    \U_d(p, q) \define \left\{ P \in \R^{d \times d}_+  \: : \:    P\1 = p, \ P^\top \1 = q  \right\}.$
	For any two vectors or matrices $p$ and $q$ having the same number of elements, we use $\< p, q\> $ to denote the dot product, i.e. elementwise multiplication then sum over all elements. We use $\|p\|_1$ to denote the sum of absolute values of the elements of $p$. 	The Bregman divergence between $p, q \in \R_+^d$ with respect to a strictly convex function $\Phi: \R_+^d \mapsto \R$ is $\D_\Phi(p, q)  \define \Phi(p) - \Phi(q) - \< \nabla \Phi(q), p - q\>.$
	We will consider the Bregman divergence with respect to the negative entropy $\Phi(p) = -H(p) \define \sum_{i} p_i (\log p_i - 1)$, where $p$ need not be a distribution. 
	When $p$ is a distribution, this corresponds to the Kullback-Leibler (KL) divergence. 
	The Bregman projection with respect to $\Phi$ of $q \in \R_+^d$ onto the set $\X$ is defined as $\P_\X \left( q \right) \define \argmin_{p \in \X} \ \D_\Phi(p, q)$.
	The Hellinger distance between $p, q \in \Sigma_n$ is defined as $h(p, q) \define \frac{1}{\sqrt2} \| \sqrt p - \sqrt q \|_2$, where $\| \cdot \|_2$ is the $l_2$-norm. We denote the square of the Hellinger distance by $h^2(p, q)$. We will often deal with \textit{marginal vectors} 
% 	\citep{globerson2008fixing,sontag2010approximate} 
	which are ordered collections of joint and marginal distributions in the form of matrices and vectors, respectively.

	\subsection{Pairwise Models}
	
	For a set of vertices, $\V = \{1, \ldots, n\}$, and edges $\EE$, a pairwise graphical model, $\G \define \{\V, \EE\}$, is a Markov random field that represents the joint distribution of variables $X_\V \define \left(X_i\right)_{i \in \V}$, taking on values from the set of states $\chi = \{0, \ldots, d - 1\}$. We assume that each vertex has at least one edge. For pairwise models, the joint distribution can be written 
    as a function of doubletons and singletons:
$	\textstyle p_\theta(x_\V) \propto \exp\left( \sum_{i \in \V} \theta_i(x_i) + \sum_{ij \in \EE} \theta_{ij}(x_i, x_j) \right).
$	
We wish to find maximum a posteriori (MAP) estimates of this model. That is, we consider the integer program:
	\begin{align} \label{eq:integer qp map}
	    \tag{Int}
	    \textstyle \max_{x_\V \in \chi^n} \quad   \sum_{i \in \V} \theta_i(x_i) + \sum_{ij \in \EE} \theta_{ij}(x_i, x_j). 
	\end{align}
	The maximization in (\ref{eq:integer qp map}) can be written as a linear program by defining a marginal vector $\Gamma$ over variable vertices $\{\Gamma_i\}_{i \in \V}$ and variable edges $\{\Gamma_{ij}\}_{ij \in \EE}$.  The vector $\Gamma_i \in \R_+^d$ represents the marginal distribution probabilities on vertex $i$ while the matrix $\Gamma_{ij} \in \R_+^{d \times d}$ represents the joint distribution probabilities shared between vertices $i$ and $j$. We follow the notation of \cite{globerson2008fixing} and denote indexing into the vector and matrix variables with parentheses, e.g. $\Gamma_{ij} (x_i,  x_j)$ for $x_i, x_j \in \chi$. The set of marginal vectors that are valid probability distributions is known as the marginal polytope and is defined as
% 	{\small
	\begin{align} \label{eq:marginal polytope}
	    {\small\M \define \left\{ \Gamma \ : \ \exists \ \mathbb P,    \begin{array}{lr}
	\mathbb P_{X_i}(x_i) = \Gamma_i(x_i), \ \forall i, x_i \\
	\mathbb P_{X_i, X_j}(x_i, x_j) = \Gamma_{ij}(x_i, x_j), \\ \quad \forall ij,  x_i, x_j \\    
	\end{array}
	    \right\} }
	\end{align}
% 	}
	We can think of $\M$ as the set of mean parameters of the model for which there exists a globally consistent distribution $\mathbb P$. We abuse notation slightly and dually view $\theta$ as a potential ``vector.''  The edge matrix $\theta_{ij} \in \R^{d\times d}$ is indexed as $\theta_{ij}(x_i, x_j)$, indicating the element at the $x_i$th row and $x_j$th column. The vertex vector $\theta_i$ is indexed as $\theta_i(x_i)$, indicating the $x_i$th element.
	The MAP problem in (\ref{eq:integer qp map}) can be shown to be equivalent to the following LP \citep{wainwright2008graphical}: 
	\begin{align*}
	    \max \quad  \< \theta, \Gamma\> \quad \st \quad \Gamma \in \M 
	\end{align*}
	where $\< \theta, \Gamma\> = \sum_{i \in \V}\sum_{x_i} \theta_i(x_i) \Gamma_{i}(x_i) + \sum_{ij \in \EE} \sum_{x_i, x_j} \theta_{ij}(x_i, x_j) \Gamma_{ij}(x_i, x_j)$.
	
	\subsection{Sherali-Adams Relaxations}
	
	The number of constraints in $\mathcal M$ is unfortunately superpolynomial  \citep{sontag2010approximate}.  This motivates considering relaxations of the marginal polytope to outer polytopes that involve fewer constraints.  For example, the \emph{local outer polytope} is obtained by enforcing consistency only on edges and vertices:
	\begin{align} \label{eq:local polytope}
	\LL_2 \define \left\{ \Gamma \geq 0 \ : \ \begin{array}{lr}
	\Gamma_i \in \Sigma_d & \forall i \in \V  \\
	\Gamma_{ij} \in \U_d(\Gamma_i, \Gamma_j)  & \forall ij \in \EE\\    
	\end{array} \right\}
	\end{align}
	Relaxations of higher orders have also been studied, in particular by \cite{sherali1990hierarchy} who introduced a hierarchy of polytopes by enforcing consistency on joint distributions of increasing order up to $n$:
	$\LL_2 \supseteq \LL_3 \supseteq \ldots \supseteq \LL_n \equiv \M$.
	The corresponding Sherali-Adams LP relaxation of order $m$ is then 
    \begin{align}\label{eq:exact}
    \tag{LP}
    \max \quad \<  \theta, \Gamma\> \quad \st \quad \Gamma \in \LL_m ,
\end{align}
    where $1 \leq m \leq n$. Because $\LL_m$ is an outer polytope of $\M$, we no longer have that the solution to (\ref{eq:exact}) recovers the true MAP solution of (\ref{eq:integer qp map}) in general. However if the solution to (\ref{eq:exact}) is integral, then $x_i = \argmax_{x} \Gamma_i(x)$ recovers the optimal solution of the true MAP problem. In this case, we say $\LL_m$ is \textit{tight}.
    
\section{ENTROPY-REGULARIZED MAP}\label{section:entropy reg SA}

In this section, we present our first main technical contribution, characterizing the approximation error in the entropy-regularized MAP problem for Sherali-Adams relaxations. In contrast to solving the exact (\ref{eq:exact}), we aim to solve the entropy-regularized LP:
\begin{align}\label{eq:reg}
    \tag{Reg}
    \min \quad \<C, \Gamma\> - \frac{1}{\eta}H(\Gamma) \quad \st \quad \Gamma \in \LL_m,
\end{align}
where $C \define - \theta$
and $H(\Gamma) = \<\Gamma, - \log \Gamma + \1\>$. 
The hyperparameter $\eta$ adjusts the level of regularization. 
Denote by $\Gamma^*_\eta$ the solution of (\ref{eq:reg}) where we omit the reference to $m$ to alleviate notation. In addition to their extensive history in inference problems, entropy-regularized LPs have arisen in a number of other fields to aid optimization when standard LP solvers are insufficient. For example, recent work in optimal transport has relied on entropy regularization to derive alternating projection algorithms \citep{cuturi2013sinkhorn,benamou2015iterative} which admit almost linear time convergence guarantees in the size of the cost matrix \citep{altschuler2017near}. Some of our theoretical results draw inspiration from these works.

\subsection{Approximation Error}
When $\LL_m$ is \textit{tight} and the solution is unique, we show that approximate solutions from solving (\ref{eq:reg}) are not necessarily detrimental because we can apply standard vertex rounding schemes to yield consistent integral solutions. It was shown by \cite{cominetti1994asymptotic}, and later refined by \cite{weed2018explicit}, that the approximation error of general entropy-regularized linear programs converges to zero at an exponential rate in $\eta$. Furthermore, it is possible to determine how large $\eta$ should be chosen in order for rounding to exactly recover the optimal solution to (\ref{eq:integer qp map}). The result is summarized in the following extension of Theorem 1 of \cite{weed2018explicit}\footnote{The entropy is defined without the linear offset in \cite{weed2018explicit}.}. 
\begin{theorem}\label{prop:approximation error}
Let $\mathcal{R}_1 = \max_{\Gamma \in \LL_m} \| \Gamma \|_1$, $\mathcal{R}_H = \max_{\Gamma, \Gamma' \in \LL_m} H(\Gamma) - H(\Gamma')$, $\mathcal{V}_m$ be the set of vertices of $\LL_m$, and $\mathcal{V}^*_m \subseteq \mathcal{V}_m$ the set of optimal vertices with respect to $C$. Let $\Delta = \min_{ V_1 \in \mathcal{V}_m \backslash \mathcal{V}^*_m, V_2 \in \mathcal{V}_m^*} \langle C, V_1 \rangle - \langle C, V_2 \rangle$ be the smallest gap in objective value between an optimal vertex and any suboptimal vertex of $\LL_m$. 
Suppose $\LL_m$ is \textit{tight} and $|\mathcal{V}^*_m | = 1$. If $\eta \geq \frac{2\mathcal{R}_1\log{64\mathcal{R}_1} + 2\mathcal{R}_1 + 2\mathcal{R}_H}{\Delta}$, the following rounded solution is a MAP assignment:
\begin{equation*}
    \left(\mathrm{round}(\Gamma_\eta^*) \right)_i := 
        \argmax_{x \in \chi} (\Gamma_\eta^*)_i(x)
\end{equation*}
% is a MAP assignment.
\end{theorem}

\begin{proof}
Define $\widetilde{C} = C + \mathbf{1} \frac{1}{\eta}$, where $\mathbf{1}$ denotes an all-ones vector with the same dimensions as $C$. If $\eta \geq \frac{4\mathcal{R}_1}{\Delta}$ then $\widetilde{\V}_m^*$, the set of optimal vertices of $\LL_m$ with respect to $\widetilde{C}$, satisfies $\widetilde{\mathcal{V}}_m^* = \mathcal{V}_m^*$ and $ \min_{ V_1 \in \mathcal{V}_m \backslash \widetilde{\mathcal{V}}^*_m, V_2 \in \widetilde{\mathcal{V}}_m^* } \langle C, V_1 \rangle - \langle C, V_2 \rangle \geq \frac{\Delta}{2}$.   If $V \in \widetilde{\mathcal{V}}_m^*$; and $V' \in \mathcal{V}_m \backslash \widetilde{\mathcal{V}}_m^*$, then $\langle \widetilde{C}, V' \rangle - \langle \widetilde{C} , V \rangle \geq \Delta - \frac{1}{\eta}\| V ' - V \|_1 \geq \frac{\Delta}{2}$. Let $\widetilde{\Delta} = \frac{\Delta}{2}$. 
If $\eta \geq \frac{\mathcal{R}_1\log{64\mathcal{R}_1} + \mathcal{R}_1 + \mathcal{R}_H}{\widetilde{\Delta}}$, and $| \widetilde{\mathcal{V}}_m^* | = 1$ then $2\mathcal{R}_1\exp\left( -\eta \frac{\widetilde{\Delta}}{\mathcal{R}_1} + \frac{\mathcal{R}_1 + \mathcal{R}_H}{\mathcal{R}_1}\right) \leq \frac{1}{32}$. And therefore, by Corollary 9 of \cite{weed2018explicit}  $\min_{ \Gamma \in \mathcal{V}^*_m } \| \Gamma - \Gamma_\eta^* \|_1 \leq \frac{1}{32}$. Since $\LL_m$ is assumed to be tight and $\widetilde{\mathcal{V}}_m^* = \mathcal{V}_m^*$ contains a single integral vertex $\Gamma^*$, the last equation implies $\mathrm{round}(\Gamma_\eta^*) = \Gamma^*$.
\end{proof}

Consequently, since $\mathcal{R}_1 \leq \sum_{j=1}^m \binom{n}{j} d^j$ and $\mathcal{R}_H \leq  \sum_{j=1}^m \binom{n}{j} \log(d^j)$\footnote{For $m= 2$ we can get tighter bounds corresponding to the number of edges in the graph $\mathcal{G}$.}, we have:
\begin{corollary}
If $\LL_m$ is tight, $|\mathcal{V}_m^*| = 1$, and $\eta \geq \frac{\log(8m n^md^m ) + 2mn^m d^m }{\Delta} $,  the rounded solution $\mathrm{round}(\Gamma_\eta^*)$ is a MAP assignment.
\end{corollary}
In general the dependence of $\Delta$ on $\eta$ suggested by Theorem \ref{prop:approximation error} is not improvable \citep{weed2018explicit}. Nevertheless, when $m=2$ and $d = 2$, since all vertices in $\mathcal{V}_2$ have entries equal to either $0, \frac{1}{2}$ or $1$---see \cite{padberg1989boolean} or Theorem 3 of \cite{weller2016tightness}---if the entries of $C$ are all integral, we have $\Delta \geq \frac{1}{2}$, thus yielding a more concrete guarantee.
The disadvantage of choosing exorbitantly large $\eta$ is that efficient computation of solutions often becomes more difficult in practice \citep{weed2018explicit,benamou2015iterative,altschuler2017near}. Thus, in practice, there exists a trade-off between computation time and approximation error that is controlled by $\eta$. We will provide a precise theoretical characterization of the trade-off in Section \ref{section::convergence_analysis}. In our guarantees, multiplying $C$ by a constant $a$ (and therefore multiplying $\Delta$ by $a$) is equivalent to multiplying $\eta$ by the same value.

\subsection{Equivalent Bregman Projection}

The objective (\ref{eq:reg}) can be interpreted as a Bregman projection. This interpretation has been explored by \cite{ravikumar2010message} as a basis for proximal updates and also \cite{benamou2015iterative} for the optimal transport problem. 
The objective is equivalent to
\begin{align*}\label{eq:proj obj}
\tag{Proj}
    \min \quad  \D_\Phi\left(\Gamma, \exp({-\eta C})\right) \quad \st \quad \Gamma \in \LL_m,
\end{align*}
where $\Phi := - H$. The derivation, based on a mirror descent step can be found in the appendix. 
The projection, however, cannot be computed in closed form in general due to the complex geometry of $\LL_m$.

 \cite{ravikumar2010message} proposed using the \textit{Bregman method} \citep{bregman1966relaxation}, which has been applied in many fields to solve difficult constrained problems \citep{benamou2015iterative,goldstein2009split,osher2005iterative,osher2011fast}, to compute $\P_{\LL_m} (\exp({-\eta C}))$ for the inner loop calculation of their proximal algorithm. While the outer loop proximal algorithm can be shown to converge at least linearly, the inner loop rate was not analyzed and the constants (possibly dependent on dimension) were not made clear. 
 Furthermore, the Bregman method is in general inexact, which makes the approximation and the effect on the outer loop unclear \citep{liu2013variational}.
 
%  Most other smooth message passing algorithms do not consider an outer proximal loop, so effectively they just solve the inner loop, which is exactly (\ref{eq:proj obj}). Notably, \cite{meshi2012convergence} gave an analysis of non-asymptotic convergence rates to the optimal value of (\ref{eq:proj obj}), but they did not consider how to recover the true MAP assignment (i.e. the solution to (\ref{eq:exact})). In the next section, we formally describe the smooth message passing inspired by \cite{ravikumar2010message} over the $\LL_2$ polytope a variant of it. This algorithm is essentially an edge-based version of the star-based algorithm analyzed by \cite{meshi2012convergence}. 
 
\section{SMOOTH MESSAGE PASSING}\label{section:bregman method SA}

We are interested in analyzing a class of algorithms closely inspired by max-sum diffusion (MSD)  as presented by \citet{werner2007linear} and the proximal updates of \cite{ravikumar2010message} to solve (\ref{eq:proj obj}) over the $\LL_2$ polytope. We describe it in detail here, with a few minor modifications and variations to facilitate theoretical analysis. In $\LL_2$, the constraints occur only over edges between vertices\footnote{Written explicitly, the constraints actually occur between any pair of vertices, but these variables play no role in the objective or constraints.}. 
Given an edge $ij \in \EE$, we must enforce the constraints prescribed by (\ref{eq:local polytope}), which is the intersection of the following sets:
\begin{align*}
\ref{proj ij} &  \quad \X_{ij\rightarrow i}  = \{ \Gamma \ : \ \Gamma_{ij} \1  =  \Gamma_i\}  \\
\ref{proj i}&  \quad \X_{ij,i}  = \{  \Gamma  \ : \ \Gamma_i^\top \1 = 1, \ \1^\top \Gamma_{ij} \1= 1 \} \\  
\ref{proj ji} &  \quad \X_{ij \rightarrow j}  = \{  \Gamma   \ : \ \Gamma_{ij}^\top  \1  = \Gamma_j \} \\
\ref{proj j} &   \quad \X_{ij,j}  = \{  \Gamma \ : \ \Gamma_j^\top \1 = 1,\ \1^\top \Gamma_{ij} \1 = 1\}.
\end{align*}
The normalization of the joint distribution $\Gamma_{ij}$ in \ref{proj i} and \ref{proj j} is actually a redundant constraint, but it  facilitates analysis as we demonstrate in Section~\ref{section::convergence_analysis}. For each of these affine constraints, we can compute the Bregman projections in closed form with simple multiplicative updates.

\begin{proposition}\label{theorem:projections}
	For a given edge ${ij} \in \EE$, the closed-form solutions of the Bregman projections for each of the above individual constraints are given below.
\begin{enumerate}[label=(\alph*)]
        \item \label{proj ij} Left consistency: If $\Gamma' = \P_{\X_{ij \rightarrow i}}(\Gamma)$, then for all $x_i, x_j \in \chi$, 
		$\Gamma_{ij}'(x_i, x_j)  \leftarrow \Gamma_{ij}(x_i, x_j) \sqrt{ \frac{\Gamma_i(x_i)}{\sum_{x} \Gamma_{ij}(x_i, x)} }$ and %,   & & \quad
		$\Gamma_{i}'(x_i)  \leftarrow 
		\Gamma_{i}(x_i) \sqrt{ \frac{\sum_{x} \Gamma_{ij}(x_i, x)}{\Gamma_i(x_i)} }$.

\item \label{proj i} Left normalization: If $\Gamma' = \P_{\X_{ij, i}}(\Gamma)$, then for all $x_i \in \chi$,
$\Gamma_i' \leftarrow \frac{\Gamma_i}{\sum_{x} \Gamma_i(x)}$ and 
 $\Gamma_{ij}' \leftarrow \frac{\Gamma_{ij}}{\sum_{x_i, x_j} \Gamma_{ij}(x_i, x_j)}$.

\item \label{proj ji} Right consistency: If $\Gamma' = \P_{\X_{ij\rightarrow j}}(\Gamma)$, then for all $x_i, x_j \in \chi$, 
		$\Gamma_{ij}'(x_i, x_j) \leftarrow \Gamma_{ij}(x_i, x_j) \sqrt{ \frac{\Gamma_j(x_j)}{\sum_{x} \Gamma_{ij}(x, x_j)} }$ and %,   & & \quad 
		$\Gamma_{j}'(x_j) \leftarrow \Gamma_{j}(x_j) \sqrt{ \frac{\sum_{x} \Gamma_{ij}(x, x_j)}{\Gamma_j(x_j)} }$.
	
		\item \label{proj j} Right normalization: If $\Gamma' = \P_{\X_{ij, j}}(\Gamma)$, then for all $x_j \in \chi$,
		$\Gamma_j' \leftarrow \frac{\Gamma_j}{\sum_{x} \Gamma_j(x)}$
		and $\Gamma_{ij}' \leftarrow \frac{\Gamma_{ij}}{\sum_{x_i, x_j} \Gamma_{ij}(x_i, x_j)}$.
	\end{enumerate}
\end{proposition}
\vspace{-.3cm}

\begin{figure}
	\begin{algorithm}[H]
		\caption{  EMP-cyclic $(C, \eta, \epsilon)$} \label{alg1}
		\begin{algorithmic}[1]
			\STATE $\Gamma \leftarrow \textsc{Normalize}(\exp({-\eta C}))$
			\STATE $k \leftarrow 1$
			\WHILE{ 
				$\max_{ij} \left\{ \begin{array}{lr}\max\{ \|\Gamma_{ij}^{(k)}\1 - \Gamma_i^{(k)} \|_1,\\   \|(\Gamma_{ij}^{(k)})^\top\1 - \Gamma_j^{(k)}\|_1 \} \end{array} \right\} \geq \epsilon$}
			
			\STATE $\Gamma \leftarrow \Gamma^{(k)}$
			% \STATE If $\mathrm{TAG} = \mathrm{CYCLIC}$:
			\FOR{$ij \in \EE$}
			\STATE $\Gamma \leftarrow (\P_{\X_{ij, j}} \circ \P_{\X_{ij \rightarrow j}} \circ \P_{\X_{ij, i}}\circ \P_{\X_{ij\rightarrow i}}) (\Gamma)$

			\ENDFOR
			
			\STATE $\Gamma^{(k + 1)} \leftarrow \Gamma$ \STATE $k \leftarrow k+ 1$
			\ENDWHILE
			\RETURN $\textsc{round}(\Gamma^{(k)})$
		\end{algorithmic}
	\end{algorithm}
\vspace{-.5cm}
	\caption{The EMP-cyclic algorithm \citep{ravikumar2010message} projects on all edges in order until the constraints are satisfied up to $\epsilon$ in $l_1$ distance. The operator $\circ$ denotes the composition of the projection operations.}
	\vspace{-.5cm}
\end{figure}
These update rules are similar to a number of algorithms throughout the literature on LP relaxations. Notably, they can be viewed as a smoothed version of MSD \citep{werner2007linear, kovalevsky1975diffusion} in that the updates enforce agreement between variables on the edges and vertices. Nearly identical smoothed updates were also initially proposed by \cite{ravikumar2010message}. As in MSD, it is common for message passing schemes derived from LP relaxations to operate on dual objective instead. We presented the primal view here as the Bregman projections lend semantic meaning to the updates and ultimately the stopping conditions in the algorithms. An equivalent dual view is presented in Appendix~\ref{sec::dual}.

% \subsection{Edge-based Message Passing}

Based on these update rules, we formally outline the algorithms we wish to analyze, which we call \textit{edge-based message passing} (EMP) for convenience.  
We consider two variants: EMP-cyclic (Algorithmic~\ref{alg1}), which cyclically applies the updates to each edge in each iteration and EMP-greedy (Algorithmic~\ref{alg2}), which applies a single projection update to only the edge with the greatest constraint violation in each iteration. We emphasize that these algorithms are not fundamentally new, but our analysis in the next section is our main contribution. EMP-cyclic \textit{is} the Bregman method, almost exactly the inner loop proposed by \cite{ravikumar2010message}.
In both variants, $\Gamma^{(1)}$ is defined as the normalized value of $\exp(-\eta C)$. 
The \textsc{GreedyEdge} operation in EMP-greedy is defined as
\begin{align*}
{\small
 \textsc{GreedyEdge}(\Gamma)  
 = \argmax_{ij \in \EE} \left\{ \begin{array}{lr} \max\{ \| \Gamma_{ij}^{(k)}\1 - \Gamma_i^{(k)}\|_1, \\  \| (\Gamma_{ij}^{(k)})^\top\1 - \Gamma_j^{(k)}\|_1 \} \end{array} \right\}
 }
\end{align*}
These procedures are then repeated again until the stopping criterion is met, which is that $\Gamma^{(k)}$ is $\epsilon$-close to satisfying the constraint that the joint distributions sum to the marginals for all edges. Both algorithms also conclude with a rounding operation.
Any fixed point of EMP must correspond to an optimal $\Gamma_\eta^*$ (see details in appendix). Computationally, EMP-greedy requires a search over the edges to identify the greatest constraint violation, which can be efficiently implemented using a max-heap \citep{nutini2015coordinate}.
% However, it may still be relatively slow in contrast to coordinate descent algorithms .
% In the next section, we will present convergence rate guarantees for both algorithms.

\begin{figure}
	\begin{algorithm}[H]
		\caption{  EMP-greedy $(C, \eta, \epsilon)$} \label{alg2}
		\begin{algorithmic}[1]
			\STATE $\Gamma \leftarrow \textsc{Normalize}(\exp({-\eta C}))$
			\STATE $k \leftarrow 1$
			\WHILE{ 
				$\max_{ij} \left\{ \begin{array}{lr}\max\{ \|\Gamma_{ij}^{(k)}\1 - \Gamma_i^{(k)} \|_1,\\   \|(\Gamma_{ij}^{(k)})^\top\1 - \Gamma_j^{(k)}\|_1 \} \end{array} \right\} \geq \epsilon$}
			
			\STATE $ij \leftarrow \textsc{GreedyEdge}(\Gamma^{(k)})$
			\IF{$\|\Gamma_{ij}^{(k)}\1 - \Gamma_i^{(k)} \|_1 > \|(\Gamma_{ij}^{(k)})^\top\1 - \Gamma_j^{(k)}\|_1$}
			\STATE $\Gamma^{(k+1)}\leftarrow (\P_{\X_{ij, i}}\circ \P_{\X_{ij\rightarrow i}}) (\Gamma^{(k)})$
			\ELSE
			\STATE $\Gamma^{(k+1)}\leftarrow  (\P_{\X_{ij, j}} \circ \P_{\X_{ij \rightarrow j}}) (\Gamma^{(k)})$
			\ENDIF
			
			\STATE $k \leftarrow k+ 1$
			\ENDWHILE
			\RETURN $\textsc{round}(\Gamma^{(k)})$
		\end{algorithmic}
	\end{algorithm}
	\vspace{-.5cm}
	\caption{The EMP-greedy algorithm selects the edge and direction with the greatest constraint violation and projects until all constraints are satisfied up to $\epsilon$ in $l_1$ distance.}
	\vspace{-.5cm}
\end{figure}

\section{THEORETICAL ANALYSIS}\label{section::convergence_analysis}

\begin{figure*}

\begin{mdframed}
% {\small
\begin{align}\label{eq:lyap}
\begin{aligned}
 L(\lambda, \xi) 
& = -  \textstyle \sum_{ij \in \EE} \sum_{x_i, x_j \in \chi} \exp\left(- \eta C_{ij}(x_i, x_j) - \lambda_{ij}(x_i) - \lambda_{ji}(x_j) - \xi_{ij}\right) \\
&  \quad -  \textstyle \sum_{i \in \V} \sum_{x \in \chi} \exp\left( -\eta C_{i}(x) - \xi_i +  \sum_{j \in N_r(i)} \lambda_{ij}(x) + \sum_{j \in N_c(i)} \lambda_{ji}(x) \right) \\
& \quad  -  \textstyle \sum_{ij \in \EE} \xi_{ij} - \sum_{i \in \V} \xi_i + \sum_{ij \in \EE} \sum_{x_i, x_j \in \chi} \exp(-\eta C_{ij}(x_i, x_j)) + \sum_{i}\sum_{x \in \chi} \exp(-\eta C_{i}(x))
\end{aligned}
\end{align}
% }
\vspace{-.5cm}
\caption{The proposed Lyapunov function. $N_r(i)$ denotes the set of neighboring vertices of $i$ where row consistency is enforced. $N_c(i)$ is the same for column consistency. The Lyapunov function $L$ can be derived from the dual objective of (\ref{eq:proj obj}). A full derivation is provided in the appendix.}
\label{fig:lyap}
\end{mdframed}
\vspace{-.5cm}
\end{figure*}
We now present our main contribution, a theoretical analysis of EMP-cyclic and EMP-greedy. This result combines two aspects. First, we present a convergence guarantee on the number of iterations sufficient to solve (\ref{eq:proj obj}), satisfying the $\LL_2$ constraints with $\epsilon>0$ error in $l_1$ distance. We note that, in finite iterations, the pseudo-marginals of EMP are not primal feasible in general due to this $\epsilon$-error.
We then combine this result with our guarantee on the approximation error in Theorem~\ref{prop:approximation error} to show a bound on the number of iterations sufficient to recover the true integral MAP assignment by rounding, assuming the LP is tight and the solution is unique. This holds with sufficient iterations and a sufficiently large regularization constant \textit{even though} the pseudo-marginals may not be primal feasible. We emphasize that these theorems are a departure from usual convergence rates in the literature \citep{meshi2012convergence,meshi2015smooth}. Prior work has guaranteed convergence in objective value to the optimum of the regularized objective (\ref{eq:proj obj}), making it unclear whether the optimal MAP assignment can be recovered, e.g. by rounding. We address this ambiguity in our results.  

We begin with the upper bound iterations to obtain $\epsilon$-close solutions, which is the result of two facts which we show. The first is that the updates in Proposition~\ref{theorem:projections} monotonically improve a Lyapunov (potential) function by an amount proportional to the constraint violation as measured via the Hellinger distance. The second is that the difference between the initial and optimal values of the Lyapunov function is bounded.

Let $\deg({\G})$ denote the maximum degree of graph $\G$ and define:
{\small
\begin{align*}
     S
    & \define \sum_{ij \in \mathcal{E}} \left[  \log \sum_{x_i, x_j \in \chi}   e^{    -\eta C_{ij}(x_i,x_j)}  + \sum_{x_i, x_j \in \chi}   \frac{\eta}{d^2} C_{ij}(x_i, x_j) \right] \\
       & + \sum_{i \in \mathcal{V}} \left[ \log \sum_{x \in \chi} e^{ -\eta C_i(x)}    +  \sum_{x \in \chi } \frac{\eta}{d} C_{i}(x) \right].
\end{align*}
}
\begin{theorem}\label{theorem:convergence}
	For any $\epsilon > 0$, EMP is guaranteed to satisfy $\|\Gamma_{ij}\1 - \Gamma_i\|_1 < \epsilon$ and $\|\Gamma_{ij}^\top\1 - \Gamma_j\|_1 < \epsilon$ for all $ij \in \EE$ in $\ceil{\frac{4\mathcal{S}_0(\deg(\G) +1)}{\epsilon^{2}}}$ iterations for EMP-cyclic and $\ceil{\frac{4 \mathcal S_0}{\epsilon^2}}$ iterations for EMP-greedy.
\end{theorem}
Here, $\mathcal{S}_0 = \min( \|\eta C / d + \exp({-\eta C})\|_1, S)$. In this theorem, we give our guarantee in terms of $l_1$ distance rather than function value convergence. As we will see, this is significant, allowing us to relate this result to Theorem~\ref{prop:approximation error} in order to derive the main result. The proof is similar in style to \cite{altschuler2017near}.
We leave the full proof for EMP-cyclic for the appendix due to a need to handle tedious edge cases, but we state several intermediate results and sketch the proof for EMP-greedy for intuition as it reveals possibly how similar message passing algorithms can be analyzed. We first introduce a Lyapunov function written in terms of dual variables $(\lambda, \xi)$, indexed by the edges and vertices to which they belong in $\LL_2$. We denote  the iteration-indexed dual variables as $(\lambda^{(k)}, \xi^{(k)})$. For a given edge $ij \in \EE$, constraints enforcing row and column consistency correspond to $\lambda_{ij}, \lambda_{ji} \in \R^m$, respectively. Normalizing constraints correspond to $\xi_i, \xi_j, \xi_{ij}\in \R$.  The Lyapunov function, $L(\lambda, \xi)$, is shown in Figure \ref{fig:lyap}.

We note that maximizing $L$ over $(\lambda, \xi)$ satisfies all constraints and yields the solution to (\ref{eq:proj obj}) by  first-order optimality conditions. We now present a result that establishes the monotone improvement in $L$ due to the updates in Proposition~\ref{theorem:projections}.
\begin{lemma}\label{lemma: improvement}
For a given edge $ij \in \EE$, let $\Gamma'$ and $(\lambda', \xi')$ denote the updated primal and dual variables after a projection from one of \ref{proj ij}--\ref{proj j} in Proposition~\ref{theorem:projections}. We have the following improvements on $L$. If $\Gamma'$ is equal to:
\vspace{-.3cm}
{\small
\begin{enumerate}[label=(\alph*)]
    \item $\P_{\X_{ij \rightarrow i}} (\Gamma)$, then $L(\lambda', \xi')- L(\lambda, \xi) = 2h^2(\Gamma_{ij} \1, \Gamma_{i})$
    \item $\P_{\X_{ij,i}} (\Gamma)$, then $L(\lambda', \xi')- L(\lambda, \xi) \geq 0$
    \item $\P_{\X_{ij \rightarrow j}} (\Gamma)$, then $L(\lambda', \xi')- L(\lambda, \xi) = 2h^2(\Gamma_{ij}^\top \1, \Gamma_j)$
    \item $\P_{\X_{ij,j}} (\Gamma)$, then $L(\lambda', \xi')- L(\lambda, \xi) \geq 0$.
\end{enumerate}}
\end{lemma}
\vspace{-.3cm}
This result shows that $L$ improves monotonically after each of the four updates in Proposition~\ref{theorem:projections}. 
% Furthermore, at every update, $L$ improves by an amount dependent on the amount of constraint violation between the joint and the marginals. In particular, the improvement is proportional to the squared Hellinger distance between the marginal distributions. 
Furthermore, at every update, $L$ improves by twice the squared Hellinger distance of the constraint violation between the joint and the marginals.

\begin{lemma}\label{lemma: max value} 
Let $\lambda^*$, $\xi^*$ denote the maximizers of $L$. The difference in function value between the optimal value of $L$ and the first iteration value is upper bounded
% \vspace{-.3cm}
\begin{align*}
    L(\lambda^*, \xi^*) - L(\lambda^{(1)}, \xi^{(1)}) \leq \mathcal{S}_0.
\end{align*}
\end{lemma}
% \vspace{-.3cm}
Turning to Theorem~\ref{theorem:convergence}, the result is obtained by observing that as long as the constraints are violated by an amount $\epsilon > 0$ (i.e., the algorithm has not terminated), then the Lyapunov function must improve by a known positive amount at each iteration. We provide a proof sketch for EMP-greedy.

\begin{proof} [Proof Sketch of Theorem \ref{theorem:convergence} for EMP-greedy]
	We now show how to combine the results of Lemma~\ref{lemma: improvement} and Lemma~\ref{lemma: max value} to obtain Theorem~\ref{theorem:convergence}. Let $k^*$ be the first iteration such that the termination condition in Algorithm \ref{alg2} holds with respect to some $\epsilon > 0$. Then, for any $k$ satisfying $1 \leq k < k^*$, we have that $\textsc{GreedyEdge}(\Gamma)$ selects $ij$ such that either $\|\Gamma_{ij}\1 - \Gamma_i\|_1 \geq \epsilon$ or $\|\Gamma_{ij}^\top \1 - \Gamma_j\|_1 \geq \epsilon$.
	
	Without loss of generality, suppose $\|\Gamma_{ij}\1 - \Gamma_i\|_1 \geq \|\Gamma_{ij}^\top \1 - \Gamma_j\|_1$. Therefore, we have
	\begin{align*}
	\frac{\epsilon^2}{4} \leq \frac{1}{4}\|\Gamma_{ij}\1 - \Gamma_i\|_1^2 \leq 2h^2(\Gamma_{ij}\1, \Gamma_{i}),
	\end{align*}
	where again $h^2(\Gamma_{ij}\1, \Gamma_i)$ denotes the squared Hellinger distance and the last inequality is the Hellinger inequality. Since $\Gamma_{ij}$ and $\Gamma_i$ are normalized for each iteration, this inequality is valid. Thus, $L$ improves by $2h^2(\Gamma_{ij}\1, \Gamma_{i})$ when $\P_{\X_{ij \rightarrow i}}$ occurs and by a non-negative amount when $\P_{\X_{ij, i}}$ occurs by Lemma~\ref{lemma: improvement}. Therefore, we can guarantee improvement of at least $\frac{\epsilon^2}{4}$ each iteration. Since the optimality gap is at most $\mathcal S_0$ by Lemma~\ref{lemma: max value}, this means the algorithm must terminate in $\ceil{\frac{4 \mathcal S_0}{\epsilon^2}}$ iterations.\end{proof}

We now turn to our main theoretical result. We combine our approximation and iteration convergence guarantees to fully characterize the convergence of EMP for $\LL_2$ to the optimal MAP assignment when the relaxation is tight and the solution is unique.

\begin{theorem}\label{theorem:combined}
Let $\eta \geq \frac{2\log(16  n^2d^2 ) + 16|\mathcal{E}|d^2  }{\min(\Delta, \frac{1}{128})} $, and $\epsilon^{-1} > (25d\deg(\G) |\EE|)^2 \max \left( \eta \|C\|_\infty,68 \right)$. If $\LL_2$ is tight and $|\mathcal{V}^*_2| = 1$, the EMP  algorithm returns a MAP assignment after 
$\ceil{\frac{4\mathcal{S}_0(\deg(\G) +1)}{\epsilon^{2}}}$ iterations for EMP-cyclic and after $\ceil{\frac{4 \mathcal S_0}{\epsilon^2}}$ iterations for EMP-greedy.

% \footnote{$\Delta$ is the gap as defined in Theorem \ref{prop:approximation error}. When the entries of $C$ are all integral, $\Delta \geq \frac{1}{2}$ thus yielding a bound independent of unknown parameters. }
\end{theorem}
When $C$ is integral, $\Delta \geq \frac{1}{2}$, yielding a bound of all known parameters.
The main technical challenge in producing this result is to relate the termination condition of EMP to the $l_1$ distance between $\Gamma^{(k)}$ and $\Gamma^*$ (the MAP assignment), as this may lie outside the polytope $\LL_2$. It does not suffice to provide convergence guarantees in function value as the goal of MAP inference is to produce integral assignments. The proof proceeds in two steps. First we show that $\mathbf{\mu}^{(k)}$ is the entropy-regularized solution to objective $C$ over a ``slack" polytope $\LL_2^{\nu^{(k)}}$. Where the slack vector $\nu^{(k)}$ corresponds to the constraint violations of $\mathbf{\mu}^{(k)}$. We use this characterization to ``project" $\mathbf{\mu}^{(k)}$ onto a nearby $\LL_2$ feasible point $\mathbf{\mu}^{(k)}(2)$. Second, we can use the properties of the primal objective to bound $\mathbf{\mu}^{(k)}(2)$ and $\mathbf{\mu}_\eta^*$. The proof is in the appendix. %    

\section{NUMERICAL EXPERIMENTS}\label{section::experiments}

\begin{figure}
	\centering
	\includegraphics[width=.45\textwidth]{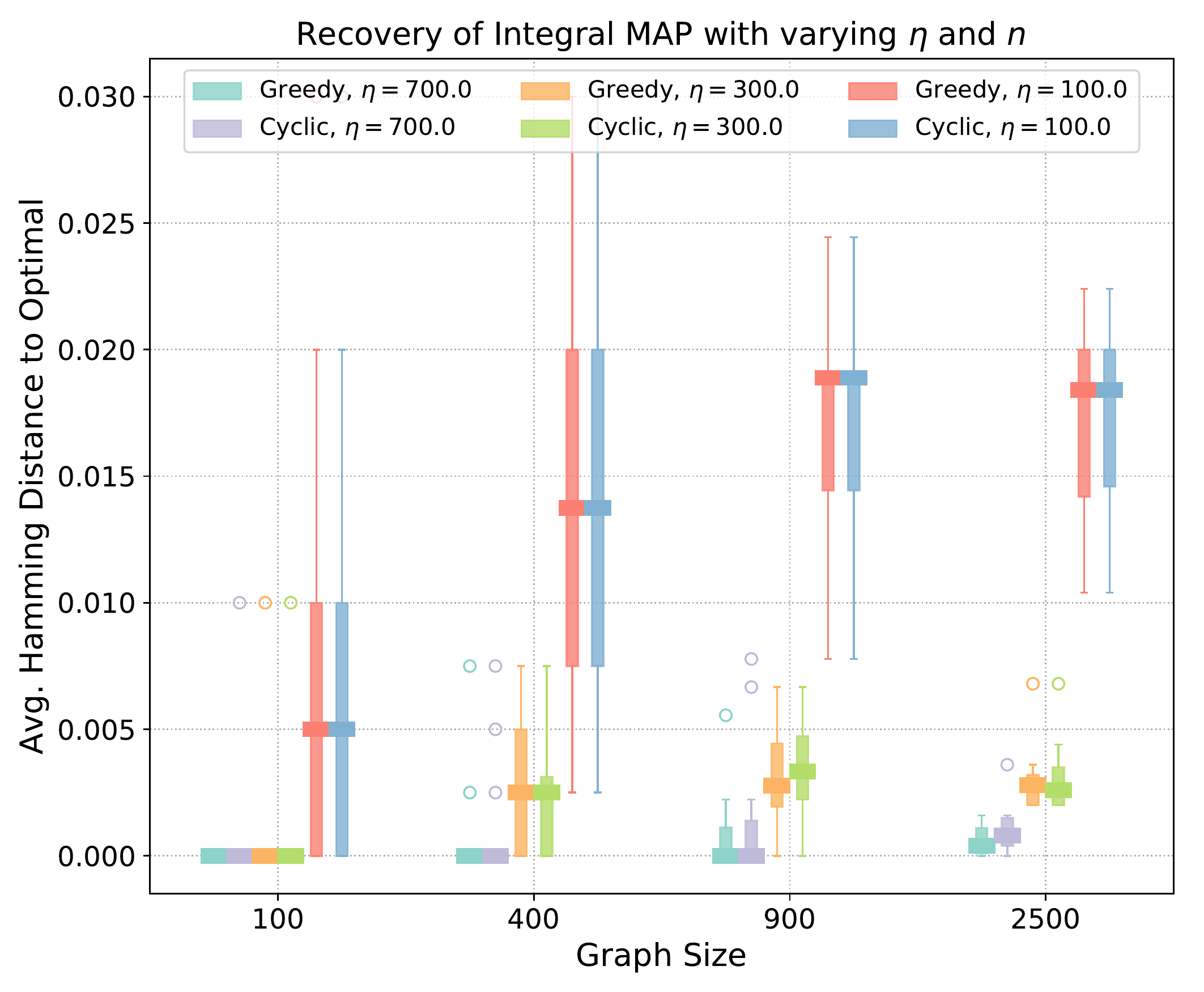}
	\caption{A box-plot showing the effect of graph size ($x$-axis) and regularization on the quality of rounded solutions for both algorithm variants after $80$ iterations. Thick horizontal bars indicate the median over 20 trials each. For large $\eta$ (cyan and purple), the true MAP is almost always recovered.}
	\label{fig:box}
\vspace{-.5cm}
\end{figure}

We illustrate our theoretical results in a practical application of the EMP algorithms. \cite{ravikumar2010message} already gave empirical evidence that the basic EMP-cyclic is competitive with standard solvers. Therefore, the objective of these experiments is to understand how graph and algorithm properties affect approximation (Theorem~\ref{prop:approximation error}) and convergence (Theorem~\ref{theorem:convergence}). We consider the family of multi-label Potts models \citep{wainwright2005map} with $d = 3$ labels on $\LL_2$. For each trial, the cost vector is $C_i(x_i) = \alpha_i(x_i)$, $\forall i, x_i$ and
\begin{align*}
    C_{ij}(x_i, x_j) & = \begin{cases}
        \beta_{ij} & x_i = x_j\\
        0 & \text{otherwise}
    \end{cases} \quad \forall ij, x_i, x_j
\end{align*}
where the parameters are random $\alpha_i(x_i) \sim\text{Unif}(-0.5, 0.5)$ and $\beta_{ij}\sim\text{Unif}\{-0.1, 0.1\}$. The graphs considered are structured as $\sqrt n \times \sqrt n$ grids \citep{globerson2008fixing, ravikumar2010message,erdogdu2017inference} and as Erd\H{o}s-R\'enyi random graphs with edge probability $p = \frac{1.1 \log n}{n}$. To evaluate recovery of the optimal MAP assignment, we first solved each graph with the ECOS LP solver \citep{domahidi2013ecos} and selected graphs that were tight. Solving the LP to find the ground-truth was the main computational bottleneck. Further details can be found in Appendix~\ref{section::experiment-details}.

\paragraph{Approximation} In Figure~\ref{fig:box}, we evaluate the effect of regularization and graph size on the quality of the nearly converged solution from EMP for over $80$ iterations on grids. The box-plots indicate that large choices of $\eta$ often yield the exact MAP solution (cyan and purple). Moderate choices still yield competitive solutions but not optimal for larger graphs (orange and green). Low choices generally give poor solutions with high spread for all graph sizes (red and blue).
\paragraph{Convergence} We then investigate the effects of regularization on convergence for both variants. Figure~\ref{fig:plot_curve} illustrates the distance of the rounded solution to the optimal MAP solution over projection steps on grids of size $n = 2500$. EMP-greedy converges sharply and varying regularization has less of an effect on its convergence rate. Finally, in Figure~\ref{fig:random}, we look at Erd\H{o}s-R\'enyi random graphs to observe the effect of the graph structure for both variants. We considered degree-limited random graphs with $\deg(\G) = 5$ and $\deg(\G) = 10$.
The figure shows convergence over projection steps for graphs of size $n = 400$. For both variants, the convergence rate deteriorates for higher degrees.

\begin{figure}
	\centering
	\includegraphics[width=.45\textwidth]{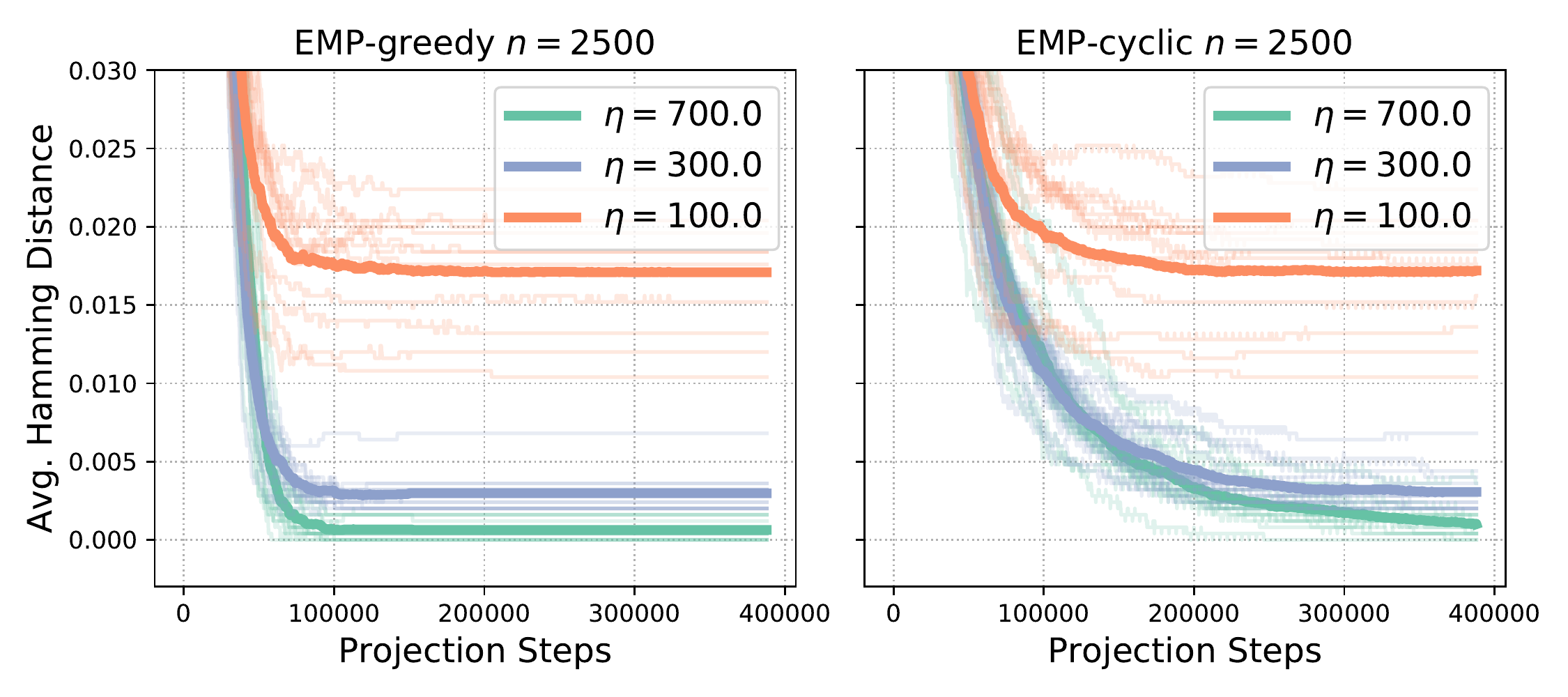}
	\caption{On grids of size $n = 2500$, convergence rates to the optimal MAP assignment of greedy and cyclic variants are shown. The  lines on each plot indicate choices of $\eta$.}
	\label{fig:plot_curve}
\vspace{-.4cm}
\end{figure}

\begin{figure}
	\centering
	\includegraphics[width=.43\textwidth]{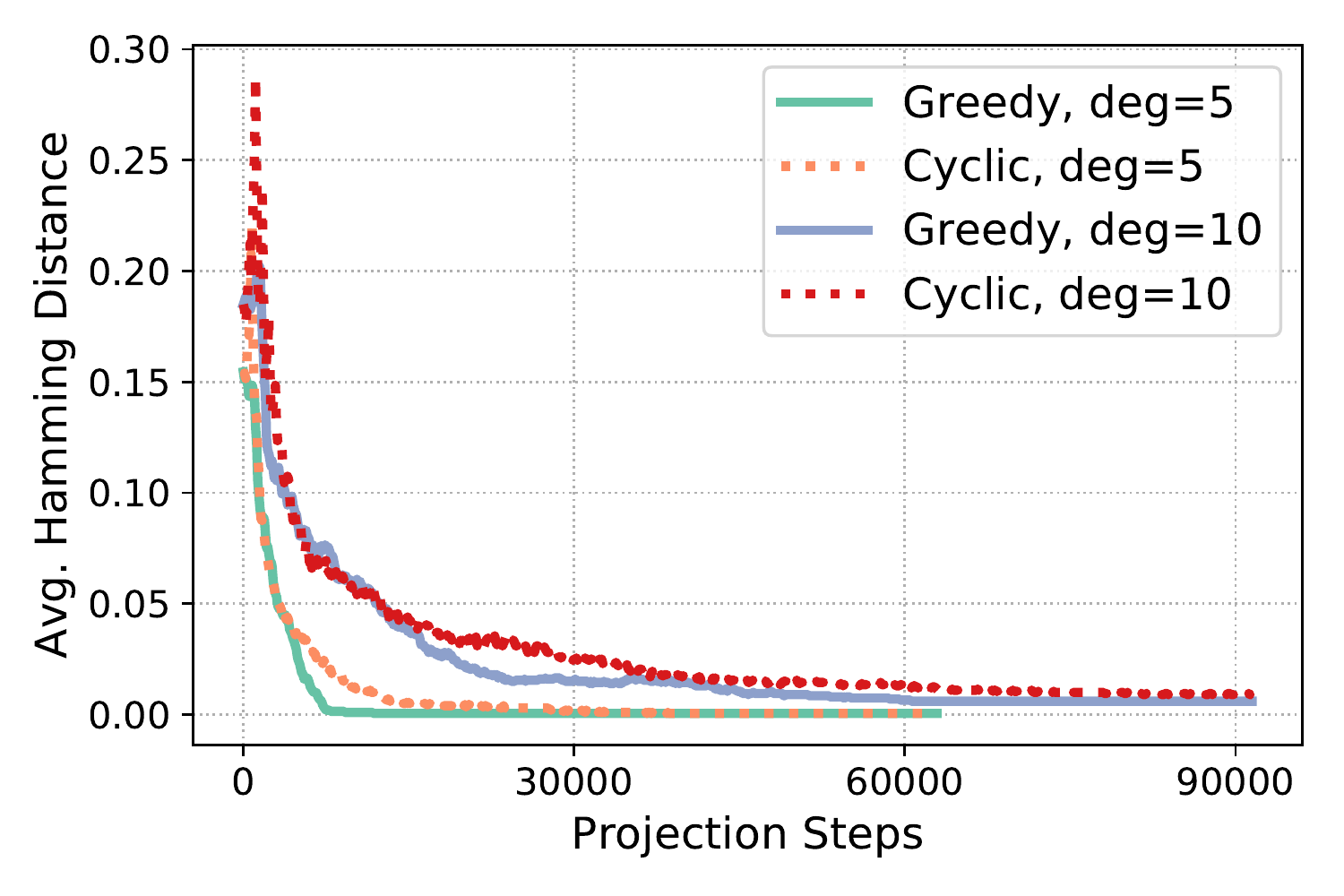}
	\caption{The algorithm variants on Erd\H{o}s-R\'enyi random graphs with $n = 400$, $\eta = 700.0$, and maximum degrees $\deg(G) = 5, 10$. The higher degree graphs (red and blue) take longer to converge to the optimal MAP assignment.}
	\label{fig:random}
	\vspace{-.5cm}

\end{figure}

\section{CONCLUSION}
In this paper, we investigated the approximation effects of entropy regularization on MAP inference objectives. We combined these approximation guarantees with a convergence analysis of an edge-based message passing algorithm that solves the regularized objective to derive guarantees on the number of iterations sufficient to recover the true MAP assignment. We also showed empirically the effect of regularization and graph propertise on both the approximation and convergence. In future work, we wish to extend the analyses and proof techniques to higher order polytopes and general block-coordinate minimization algorithms.
\pagebreak
\paragraph{Acknowledgements} 
We thank the anonymous reviewers and Marco Pavone for their invaluable feedback.
%\section*{References}
% \bibliographystyle{apalike}
\bibliographystyle{plainnat}
\bibliography{sherali}
\raggedbottom
\pagebreak

\appendix

\onecolumn

\section{Bregman Projection Derivation}

The objective (\ref{eq:reg}) can be equivalently interpreted as a Bregman projection. This interpretation has been explored by \cite{ravikumar2010message} as a basis for proximal updates and also \cite{benamou2015iterative} for the optimal transport problem. Here, we review the transformation because it is central to the algorithm of \cite{ravikumar2010message}, upon which our main theoretical results are based.

By definition of the Bregman projection with respect to the negative entropy, $\Phi= -H$, we have 
\begin{align*}
\D_\Phi(\Gamma, \1) & =_+ \< \Gamma , \log \Gamma - \1 \> - \< \log \1, \Gamma - \1\> \\
& = -H(\Gamma)
\end{align*}
where $\1$ is a vector of ones of the same size as the marginal vector and $=_+$ denotes the two sides are equal up to a constant. Substituting this into (\ref{eq:reg}) and multiplying through by $\eta$ yields the objective:
\begin{align*}
    % \min \ \<C, \Gamma\> - \frac{1}{\eta}H(\Gamma) \quad \st \ \Gamma \in \LL_m & & \iff & &
    \min \quad \eta \< C, \Gamma\> +  \D_{\Phi} \left(\Gamma, \1 \right) \quad \st  \ \Gamma \in \LL_m.
\end{align*}
Note the similarity to a projected mirror descent update over $\LL_m$ starting from $\1$ \citep{nemirovsky1983problem,bubeck2015convex}.
Using this insight and performing a single gradient update in the dual, we can transform the problem into a single Bregman projection of the vector. The unprojected marginal vector $\Gamma'$ satisfies 
\begin{align*}
    \nabla \Phi(\Gamma') = \nabla \Phi(\1)  - \eta C,
\end{align*} 
where $\nabla \Phi(\Gamma) = -\nabla H(\Gamma) = \log \Gamma$ is the dual map and $(\nabla \Phi)^{-1}(\Gamma) = \nabla \Phi^* (\Gamma) = \exp(\Gamma)$ is the inverse dual map. We have $\Gamma' = \exp(-\eta C)$ and  the solution to the mirror descent update is $\P_{\LL_m}(\exp(-\eta C))$. Therefore it is sufficient to solve the following Bregman projection problem:
\begin{align*}
% \tag{Proj}
    \min \quad  \D_\Phi\left(\Gamma, \exp({-\eta C})\right) \quad \st \quad \Gamma \in \LL_m
\end{align*}

The projection, however, cannot be computed in closed form due to the complex geometry of $\LL_m$.  Sinkhorn-like algorithms such as those used in \cite{cuturi2013sinkhorn} are unavailable because the transportation polytopes $\U_d(\Gamma_i, \Gamma_j)$ are dependent on variables $\Gamma_i$ and $\Gamma_j$ which are also involved in the projection operation.

\section{Derivation of EMP Update Rules}

	We present the derivations of the update rules similar to \cite{ravikumar2010message} for a given edge $ij \in \EE$ based on the Bregman projections onto the individual constraint sets $\X_{ij \rightarrow i}$, $\X_{ij, i}$, $\X_{ij \rightarrow j}$, $\X_{ij, j}$. We refer the reader to \cite{ravikumar2010message} for the original algorithm and derivation.
	We derive only the first two projections; the last two can be found by exchanging the indices.
	
	\begin{itemize}
		\item [\ref{proj ij}]
		For the projection $\Gamma' = \P_{\X_{ij \rightarrow i}} (\Gamma)$, where
		\begin{align*}
		\X_{ij\rightarrow i} & = \{ \Gamma \ : \ \Gamma_{ij} \1  =  \Gamma_i\},
		\end{align*}
		there are no constraints on any edges or vertices other than $ij$ and $i$. Therefore,  $\forall k \neq i$, $\Gamma_k' = \Gamma_k$. Similarly, $\forall k\ell \neq ij$, $\Gamma_{k\ell}' = \Gamma_{k \ell}$.
		
		The Lagrangian of the projection is given in terms of primal variables $\Gamma$ and dual variables $\alpha$: 
		\begin{align*}
		\L(\Gamma', \alpha) & = \sum_{x_i, x_j} \Gamma'_{ij}(x_i, x_j) \left( \log \frac{\Gamma'_{ij}(x_i, x_j)}{\Gamma_{ij}(x_i, x_j)}- 1  \right) 
		+ \sum_{x_i} \Gamma'_i(x_i) \left( \log \frac{\Gamma'_i(x_i)}{ \Gamma_i(x_i)} - 1 \right) + \alpha^\top \left(\Gamma'_{ij}\1 - \Gamma'_i \right) \\
		& =\sum_{x_i, x_j} \Gamma'_{ij}(x_i, x_j) \left( \log \frac{\Gamma'_{ij}(x_i, x_j)}{\Gamma_{ij}(x_i, x_j)}- 1  + \alpha({x_i}) \right) 
		+ \sum_{x_i} \Gamma'_i(x_i) \left( \log \frac{\Gamma'_i(x_i)}{ \Gamma_i(x_i)} - 1 - \alpha(x_i)\right).
		\end{align*}
		By the first-order optimality condition, the primal solution in terms of the dual variables is
		\begin{align*}
		\Gamma_{ij}' (x_i, x_j) & = \Gamma_{ij} (x_i, x_j) e^{-\alpha(x_i)} \\
		\Gamma_{i}'(x_i) & = \Gamma_i (x_i) e^{\alpha(x_i)}.
		\end{align*}
		Substituting this solution back in to the Lagrangian, we have
		\begin{align*}
		\L(\alpha) = -\sum_{x_i, x_j} \Gamma_{ij}(x_i, x_j) e^{-\alpha(x_i)} - \sum_{x_i} \Gamma_i(x_i) e^{\alpha(x_i)}.
		\end{align*}
		Again, by the first-order optimality condition, the dual solution is
		\begin{align*}
		\alpha^*(x_i) = \frac{1}{2}\log \frac{\sum_{x_j} \Gamma_{ij}(x_i, x_j)}{\Gamma_i(x_i)}.
		\end{align*}
		Substituting this value for $\alpha^*$ into the primal solution yields the desired result.

		\item [\ref{proj i}] Again, for the projection onto 
		\begin{align*}
		\X_{ij,i} & = \{  \Gamma \ : \ \Gamma_i^\top \1 = 1, \ \1^\top \Gamma_{ij} \1= 1 \},
		\end{align*}
		only $\Gamma_i$ and $\Gamma_{ij}$ are affected. $\X_{ij, i}$
		enforces that the variables $\Gamma_{ij}$ and $\Gamma_i$ each sum to one. It is well known and easy to show that the Bregman projection with respect to the negative entropy is simply the $\Gamma_{ij}$ and $\Gamma_i$ normalized by their sums. This normalization can also be written as a multiplicative update of the same form by observing that
		\begin{align*}
		\Gamma'_{ij}(x_i, x_j) & = \Gamma'_{ij}(x_i, x_j) e^{-\xi_{ij}^*} \\
		\Gamma'_{i} (x_i) & = \Gamma'_{i}(x_i) e^{-\xi_i^*},
		\end{align*}
		where $\xi_{ij}^* = \log \sum_{x_i, x_j} \Gamma_{ij}(x_i, x_j)$ and $\xi_i^* = \log \sum_{x_i} \Gamma_i(x_i)$. Again, these can be derived via the Lagrangian.
	\end{itemize}

\section{Extensions of EMP}

\subsection{Dual EMP}\label{sec::dual}

We may also equivalently interpret the multiplicative updates in Algorithm~\ref{alg1} and Algorithm~\ref{alg2} as additive updates of the dual variables. The dual interpretation is consistent with past work in dual MAP algorithms \citep{sontag2011introduction} and may be more practical  to avoid numerical issues in implementation. Instead of tracking the primal variables $\Gamma$, we track a sum of the dual variables with $\zeta$ for each vertex and edge. Enforcing consistency between a given joint distribution and its marginals in \ref{proj ij} yields updated dual variable sums
\begin{align*}
\zeta'_{ij}(x_i, x_j)  \leftarrow \zeta_{ij}(x_i, x_j) - \alpha^*(x_i)  &  &
\zeta'_{i}(x_i)  \leftarrow \zeta_{i}(x_i) + \alpha^*(x_i),
\end{align*}
where again $\alpha^*(x_i) = \frac{1}{2}\log \frac{\sum_{x_j} \Gamma_{ij}(x_i, x_j)}{\Gamma_i(x_i)}$. The same is done for the vertex $j$ in \ref{proj ji} with indices exchanged. The normalization step in \ref{proj i} yields
\begin{align*}
\zeta'_{ij}(x_i, x_j)  \leftarrow \zeta_{ij}(x_i, x_j) - \xi^*_{ij} & &
\zeta'_i(x_i, x_j)  \leftarrow \zeta_i(x_i) - \xi^*_{i},
\end{align*} 
where $\xi_{ij}^* = \log \sum_{x_i, x_j} \Gamma_{ij}(x_i, x_j)$ and $\xi_i^* = \log \sum_{x_i} \Gamma_i(x_i)$. Again, the same is done for \ref{proj j}. The primal marginal vector is recovered with
\begin{align*}
\Gamma = \exp(-\eta C + \zeta ).
\end{align*}
We will later make explicit the dual formulation as it will aid in the theoretical analysis.

\subsection{Clique Constraints}

The version of EMP presented in the paper is for the $\LL_2$ local polytope, which enforces only pairwise consistency among the variables with edges, but this can be fairly easily extended. In this section, we discuss higher order pseudo-marginals and their constraints. Consider the polytope that enforces consistency on all subsets of $\V$ of size $k$ and below, denoted by $\C$. We use the notation of \cite{meshi2012convergence}. The constraint set is written as
\begin{align} \label{eq:lk polytope}
\LL_\C \define \left\{ \Gamma \geq 0 \: : \: \begin{array}{lr}
\Gamma_i \in \Sigma_m & \forall i \in \V  \\
\Gamma_i(x_i) = \sum_{x_{c \setminus i} } \Gamma_c & \forall x_i \in \chi, i \in c, c  \in \C,\\    
\end{array} \right\}.
\end{align}
where $x_{c \setminus i}$ denotes a marginalization over all variables except $i$. For convenience, we may also now account for higher-order interactions in the model itself:
\begin{align*}
\max \quad \sum_{c \in \C} \sum_{x_c \in \chi^k} \theta_{c}(x_c) \Gamma_c(x_c) + \sum_{i \in \V} \sum_{x_i \in \chi} \theta_i(x_i) \Gamma_i(x_i) \quad \st \quad \Gamma \in \LL_\C
\end{align*}
The projection operation in (\ref{eq:proj obj}) is the same for $C = -\theta$. Analogous update rules to Proposition~\ref{theorem:projections} can derived with exactly the same procedure. For a given subset $c \in \C$ and vertex $i$, we have that $\Gamma' = \P_{ij \rightarrow i} (\Gamma)$ constitutes the update
\begin{align*}
\Gamma_c'(x_c) = \Gamma_{c}(x_c) \sqrt{ \frac{ \Gamma_i(x_i) }{ \sum_{x_{c \setminus i }}\Gamma_{c}(x_c) } } \\
\Gamma_i'(x_i) = \Gamma_{c}(x_c) \sqrt{ \frac{ \sum_{x_{c \setminus i }}\Gamma_{c}(x_c) }{ \Gamma_i(x_i) } }. 
\end{align*}
The normalization updates are identical as well. As in the presented EMP algorithm, we can design greedy and cyclic algorithms around these update equations. The theoretical analysis in Section~\ref{section::convergence_analysis} will focus on the case with edges only. We leave the general analysis of $\LL_\C$ for future work.

\section{Omitted Proofs and Derivations from Section~\ref{section::convergence_analysis}}\label{appendix:convergence proofs}

\subsection{Derivation of the Lyapunov function (\ref{eq:lyap})}
For convenience, $L$ is restated here:
\begin{align}
\begin{aligned}
L(\lambda, \xi) & = - \sum_{ij \in \EE} \sum_{x_i, x_j \in \chi} \exp\left(- \eta C_{ij}(x_i, x_j) - \lambda_{ij}(x_i) - \lambda_{ji}(x_j) - \xi_{ij}\right) \\
& \quad - \sum_{i \in \V} \sum_{x \in \chi} \exp\left( -\eta C_{i}(x) - \xi_i +  \sum_{j \in N_r(i)} \lambda_{ij}(x) + \sum_{j \in N_c(i)} \lambda_{ji}(x) \right) \\
& \quad  - \sum_{ij \in \EE} \xi_{ij} - \sum_{i \in \V} \xi_i + \sum_{ij \in \EE} \sum_{x_i, x_j \in \chi} \exp(-\eta C_{ij}(x_i, x_j)) + \sum_{i}\sum_{x \in \chi} \exp(-\eta C_{i}(x)).
\end{aligned}
\end{align}
The Lagrangian of (\ref{eq:proj obj}) with primal variables $\Gamma$ and dual variables $(\lambda,\xi)$ can be written as
\begin{align*}
\mathcal L(\Gamma, \lambda, \xi)  & =  \D_\Phi(\Gamma, \exp(-\eta C)) + \sum_{ij} \left( \lambda_{ij}^\top ( \Gamma_{ij} \1 - \Gamma_i) + \lambda_{ji}^\top (\Gamma_{ij}^\top \1 - \Gamma_i)  \right)  \\
& \quad  +  \sum_{ij} \xi_{ij} (\1 ^\top \Gamma_{ij} \1 - 1) +  \sum_{i} \xi_i(\Gamma_i ^\top \1  - 1),
\end{align*}
where
\begin{align*}
\D_{\Phi}(\Gamma, \exp(-\eta C)) & = \sum_{ij}\sum_{x_i, x_j} \Gamma_{ij}(x_i, x_j) \left(\log \Gamma_{ij}(x_i, x_j) + \eta C_{ij}(x_i, x_j) - 1 \right) \\
& \quad + \sum_{i}\sum_{x} \Gamma_{i}(x) \left(\log \Gamma_{i}(x) + \eta C_{i}(x) - 1 \right) \\
& \quad + \sum_{ij} \sum_{x_i, x_i} \exp(-\eta C_{ij}(x_i, x_j)) + \sum_{i}\sum_{x} \exp(-\eta C_i(x)).
\end{align*}
The partial derivatives with respect to $\Gamma_{ij}(x_i, x_j)$ and $\Gamma_i(x)$ are given by
\begin{align*}
\frac{\partial \L}{\partial \Gamma_{ij}(x_i, x_j)} & = \log \Gamma_{ij}(x_i, x_j) + \eta C_{ij}(x_i, x_j) + \lambda_{ij}(x_i) + \lambda_{ji} (x_j) + \xi_{ij} \\
\frac{\partial \L}{\partial \Gamma_{i}(x)} & = \log \Gamma_{i}(x) + \eta C_{i}(x) + \xi_{i} - \sum_{j \in N_r(i)} \lambda_{ij}(x_i) + \sum_{j \in N_c(i)}\lambda_{ji} (x_j).
\end{align*}
Setting the derivatives to zero gives the solution $\Gamma$ in terms of the dual variables: 
%%%%%%%%%%%%%%%%%%%%%
%We further have the following relationship between the primal and dual variables:
\begin{align*}
\Gamma_{ij}(x_i, x_j) & = \exp\left(- \eta C_{ij}(x_i, x_j) - \lambda_{ij}(x_i) - \lambda_{ji}(x_j) - \xi_{ij}\right) \\
\Gamma_i(x) & = \exp\left( -\eta C_{i}(x) - \xi_i +  \sum_{j \in N_r(i)} \lambda_{ij}(x) + \sum_{j \in N_c(i)} \lambda_{ji}(x) \right).
\end{align*}
By substituting $\Gamma$ in $\L$, we obtain the Lyapunov function $L$.

\subsection{Proof of Lemma~\ref{lemma: improvement}}

In this section we prove Lemma \ref{lemma: improvement}. We restate the result for the reader's convenience. 

\begin{lemma}\label{lemma: appendix}
	For a given edge $ij \in \EE$, let $\Gamma'$ and $(\lambda', \xi')$ denote the updated primal and dual variables after a projection from one of \ref{proj ij}--\ref{proj j} in Proposition~\ref{theorem:projections}. We have the following improvements on $L$. If $\Gamma'$ is equal to:
	\begin{enumerate}[label=(\alph*)]
		\item $\P_{\X_{ij \rightarrow i}} (\Gamma)$, then $L(\lambda', \xi')- L(\lambda, \xi) = 2h^2(\Gamma_{ij} \1, \Gamma_{i})$
		\item $\P_{\X_{ij,i}} (\Gamma)$, then $L(\lambda', \xi')- L(\lambda, \xi) \geq 0$
		\item $\P_{\X_{ij \rightarrow j}} (\Gamma)$, then $L(\lambda', \xi')- L(\lambda, \xi) = 2h^2(\Gamma_{ij}^\top \1, \Gamma_j)$
		\item $\P_{\X_{ij,j}} (\Gamma)$, then $L(\lambda', \xi')- L(\lambda, \xi) \geq 0$.
	\end{enumerate}
\end{lemma}

\begin{proof}

Let $L$ and $L'$ denote the values of the Lyapunov function before and after the projection in each case.
   
\ref{proj ij} \quad Due to the projection $\Gamma' =  \P_{\X_{ij\rightarrow i}}(\Gamma)$, only $\Gamma_{ij}$ and $\Gamma_i$ change values.
    \begin{align*}
        L' - L & = \sum_{x_i, x_j} \left( \Gamma_{ij}(x_i, x_j) - \Gamma_{ij}'(x_i, x_j)\right) + \sum_{x} \left( \Gamma_{i}(x) - \Gamma_{i}'(x)\right) \\
        & = \sum_{x_i, x_j} \Gamma_{ij}(x_i, x_j)  \left( 1- \sqrt{\frac{\Gamma_{i} (x_i)}{\sum_{x'}\Gamma_{ij}(x_i, x')}}\right) + \sum_{x} \Gamma_{i}(x) \left( 1- \sqrt{\frac{\sum_{x'}\Gamma_{ij}(x, x')}{\Gamma_{i} (x_i)}}\right) \\
        & = \left\|\sqrt{\Gamma_{ij} \1} - \sqrt{\Gamma_i} \right\|_2^2 = 2h^2(\Gamma_{ij}\1, \Gamma_i).
    \end{align*}
\ref{proj i} \quad Due to the projection $\Gamma' =  \P_{\X_{ij\rightarrow i}}(\Gamma)$ change, again only $\Gamma_{ij}$ and $\Gamma_i$, but they are simply normalized. From the derivation of the updates, we can see that only dual variables $\xi_i$ and $\xi_{ij}$ are updated in order for the normalization to occur. We have, from the update rule in Proposition \ref{theorem:projections}
\begin{align*}
\xi_{ij}' & = \xi_{ij} - \log \sum_{x_i, x_j} \Gamma_{ij}(x_i, x_j) \\
\xi_{i}' & = \xi_i - \log \sum_{x} \Gamma_i(x).
\end{align*}
The improvement on the Lyapunov function can then be written as
\begin{align*}
L' - L & = \sum_{x_i, x_j} \Gamma_{ij}(x_i, x_j) \left(1 - \exp(\xi_{ij}' - \xi_{ij})\right) + \sum_{x} \Gamma_i(x) \left(1 - \exp(\xi_i' - \xi_i)\right) \\
& \quad  + \xi_{ij}' - \xi_{ij} + \xi_i' - \xi_i \\
& = \sum_{x_i, x_j} \Gamma_{ij}(x_i, x_j) - \log \sum_{x_i, x_j} \Gamma_{ij}(x_i, x_j) - 1 \\
& \quad  + \sum_{x} \Gamma_i(x)- \log \sum_{x} \Gamma_i(x) - 1,
\end{align*}
where the second equality uses the fact that $\Gamma_{ij}'$ and $\Gamma_{i}$ both sum to one.
This last expression can be shown to be non-negative by recognizing the classical inequality $x - \log x - 1 \geq 0$ for all $x > 0$.

\ref{proj ji} \quad The proof of improvement is identical to \ref{proj ij}; however, we replace vertex $i$ with $j$ and all row sums $\Gamma_{ij}\1$ with column sum $\Gamma_{ij}^\top \1$.

\ref{proj j} \quad The proof of improvement is identical to \ref{proj i}, but we replace $i$ with $j$ for the vertex marginal normalization.

\end{proof}

\subsection{Fixed points of EMP}

We start this section by noting that all fixed points of EMP correspond to valid (constraint satisfying) primal solutions and therefore must equal global optima of the dual function. 

First note that any fixed point of EMP corresponds to a candidate solution all whose constraints are satisfied. Indeed, at optimality $\lambda^*, \xi^*$ satisfy:
\begin{align*}
   \left( \Gamma_\eta^*\right)_{ij}(x_i, x_j) &= \exp\left(-\eta C_{ij}(x_i,x_j) - \lambda_{ij}^*(x_i) - \lambda_{ji}^*(x_j) - \xi^*_{ij}     \right) \\
    \left(\Gamma_\eta^*\right)_i(x_i) &= \exp\left( -\eta C_{i}(x_i) - \xi^*_i +  \sum_{j \in N_r(i)} \lambda^*_{ij}(x_i) + \sum_{j \in N_c(i)} \lambda^*_{ji}(x_i) \right),
\end{align*}
with $\Gamma^*_\eta \in \LL_2$. Since all constraints are satisfied, for all projection types $\mathcal{P}$ in  Lemma \ref{lemma: improvement}, $\mathcal{P}(\Gamma_\eta^*) = \Gamma_\eta^*$. 

For the converse, we proceed by contradiction. Let $\Gamma$ be a fixed point of EMP. As such, all the normalization constraints (ensuring the edge and node distributions each sum to one) must be satisfied. Assume then that a constraint of type (a) or (c) is not satisfied. Without loss of generality let $ij \rightarrow i$ be the unsatisfied constraint. As a consequence of \ref{lemma: improvement}, the Lyapunov objective can be strictly increased by performing the corresponding Bregman projection, and therefore EMP couldn't have possibly be at a fixed point. We summarize these observations in the following proposition:

\begin{proposition}
All maxima of $L(\lambda, \xi)$ are fixed points of EMP and all fixed points of EMP are maxima of $L(\lambda, \xi)$.
\end{proposition}

\subsection{Proof of Lemma~\ref{lemma: max value}}
In this section we prove Lemma \ref{lemma: max value}, we restate it here for readability:

\begin{lemma}\label{lemma: max value_appendix} 
Let $\lambda^*$, $\xi^*$ denote the maximizers of $L$. The difference in function value between the optimal value of $L$ and the value at the        first iteration is upper bounded as
\begin{align*}
    L(\lambda^*, \xi^*) - L(\lambda^{(1)}, \xi^{(1)}) \leq \min( \|\eta C / d + \exp({-\eta C})\|_1, S).
\end{align*}
\end{lemma}

\begin{proof}
	We start by showing the upper bound:
	\begin{align}\label{eq::upper_bound_1}
	L(\lambda^*, \xi^*) - L(\lambda^{(1)}, \xi^{(1)})\leq \| \eta C / d + \exp({-\eta C})\|_1.
	\end{align}
	
We have that $(\lambda, \xi) = (0, 0)$ when $\Gamma = e^{-\eta C}$ before any updates to the primal variables. By Lemma \ref{lemma: improvement}, $L(0, 0, 0) \leq L(\lambda^{(1)},\xi^{(1)})$. Then we have 
\begin{align*}
    L(\lambda^*, \xi^*) - L(\lambda^{(1)},\xi^{(1)}) \leq  L(\lambda^*, \xi^*) - L(0, 0) \leq L(\lambda^*, \xi^*).
    %\leq \| \exp({-\eta C})\|_1
\end{align*}

We may establish an upper bound on $L(\lambda^*, \xi^*)$ by finding a feasible point in the primal objective (\ref{eq:proj obj}). It is easy to verify that $\Gamma$ is in $\LL_2$ if $\forall ij \in \EE$ and $\forall i \in \V$,  $\Gamma_{ij}(x_i, x_j) = \frac{1}{d^2}$ and $\Gamma_i(x_i) = \frac{1}{d}$.
    With this choice of $\Gamma$, the value of (\ref{eq:proj obj}) is
\begin{align*}
    \D_{\Phi}(\Gamma, \exp(^{-\eta C})) & = \sum_{ij \in \EE} (\eta \E_U [C_{ij}] - 1 - \log d^2) + \sum_{i \in \V} (\eta \E_U[ C_{i}] - 1 - \log d) \\
    & \quad + \sum_{ij \in \EE} \sum_{x_i, x_j \in \chi} \exp(-\eta C_{ij}(x_i, x_j)) + \sum_{i}\sum_{x \in \chi} \exp(-\eta C_{i}(x)) \\
    & \leq \|\frac{\eta}{d} C + \exp({-\eta C})\|_1 - (|\V| + |\EE|)(\log d + 1) ,
\end{align*}
where $\E_U$ denotes the uniform distribution. where the last inequality follows from the fact that $C_{ij}(0, 0) = C_i(0) = 0$.
Therefore,
\begin{align*}
L(\lambda^*, \xi^*) -L(\lambda^{(1)},\xi^{(1)}) \leq  L(\lambda^*, \xi^*) \leq  \|\frac{\eta}{d} C + \exp({-\eta C})\|_1 - (|\V| + |\EE|)(\log d + 1).
\end{align*}

We now proceed to show the following (direct) bound on  $L(\lambda^*, \xi^* ) - L(\lambda^{(1)}, \xi^{(1)})$:

\begin{align*}
    L(\lambda^*, \xi^* ) - L(\lambda^{(1)}, \xi^{(1)}) &\leq \sum_{ij \in \mathcal{E}} \left[  \log\left( \sum_{x_i, x_j \in \chi}   \exp\left(    -\eta C_{ij}(x_i,x_j) \right) \right) + \sum_{x_i, x_j \in \chi}   \frac{\eta}{4} C_{ij}(x_i, x_j) \right] + \\
    &\sum_{i \in \mathcal{V}} \left[ \log\left( \sum_{x \in \chi} \exp\left( -\eta C_i(x)     \right)            \right) +  \sum_{x \in \chi } \frac{\eta}{2} C_{i}(x) \right].
\end{align*}

We work under the assumption that at any time $k$, all the component distributions of $\Gamma^{(k)}$ are normalized so its entries sum to $1$. Notice that in this case
\begin{align*}
L(\lambda^*, \xi^*) - L(\lambda^{(1)},\xi^{(1)}) = \sum_{ij \in \mathcal{E}} \xi^{(1)}_{ij} - \xi^{*}_{ij} + \sum_{i \in \mathcal{V}} \xi_i^{(1)} - \xi^{*}_i.
\end{align*}
If we initialize our algorithm to $\lambda^{(1)} = 0$, and $\xi^{(1)}$ be the normalization factors corresponding to this choice of $\lambda$, then

\begin{equation*}
    \sum_{ij \in \mathcal{E}} \xi^{(1)}_{ij}  + \sum_{i \in \mathcal{V}} \xi_i^{(1)} = \sum_{ij \in \mathcal{E}} \log\left( \sum_{x_i, x_j \in \chi}   \exp\left(    -\eta C(x_i,x_j) \right) \right) + \sum_{i \in \mathcal{V} } \log\left( \sum_{x \in \chi} \exp\left( -\eta C(x)     \right)            \right).
\end{equation*}

Notice that at optimality $\lambda^*, \xi^*$, for all $ij \in \mathcal{E}$ and, for all $x_i,x_j$,

\begin{equation*}
    \exp\left( -\eta C_{ij}(x_i,x_j) -\lambda^*_{ij}(x_i) -\lambda^*_{ji}(x_j) - \xi^*_{ij}  \right) = \left( \Gamma_{\eta}^* \right)_{ij}(x_i, x_j) \in [0,1].
\end{equation*}

And for all $i \in \mathcal{V}$ and for all $x$,

\begin{equation*}
    \exp\left(  -\eta C_i(x) - \xi_i^* + \sum_{j \in N_r(i)} \lambda^*_{ij}(x) + \sum_{j \in N_c(i)} \lambda^*_{ji}(x)     \right) = \left(  \Gamma_\eta^*     \right)_{i}(x) \in [0,1].
\end{equation*}

Therefore, for all $ij \in \mathcal{E}$ and for all $x_i, x_j$:
\begin{equation}\label{equation::negative_1}
    -\eta C_{ij}(x_i,x_j) -\lambda^*_{ij}(x_i) -\lambda^*_{ij}(x_j) - \xi^*_{ij}  \leq 0
\end{equation}
For all $i\in \mathcal{V}$ and for all $x$:
\begin{equation}\label{equation::negative_2}
-\eta C_i(x) - \xi^*_i + \sum_{j \in N_r(i)} \lambda^*_{ij}(x) + \sum_{j \in N_c(i)} \lambda^*_{ji}(x)  \leq 0
\end{equation}
Summing Equations (\ref{equation::negative_1}) and (\ref{equation::negative_2}) over all $ij \in \mathcal{E}$, $i \in \mathcal{V}$ and $x_i, x_j, x \in \chi$ yields:
\begin{equation}\label{equation::upper_bound}
  -  \sum_{ij \in \mathcal{E}} \xi^*_{ij} - \sum_{i \in \mathcal{V}} \xi^*_i \leq  \sum_{ij \in \mathcal{E}} \sum_{x_i, x_j \in \chi} \frac{\eta}{d^2} C_{ij}(x_i, x_j)  + \sum_{i \in \mathcal{V}} \sum_{x \in \chi } \frac{\eta}{d} C_{i}(x)
\end{equation}
And, therefore,
\begin{align}
    L(\lambda^*, \xi^* ) - L(\lambda^{(1)}, \xi^{(1)}) &\leq \sum_{ij \in \mathcal{E}} \left[  \log\left( \sum_{x_i, x_j \in \chi}   \exp\left(    -\eta C(x_i,x_j) \right) \right) + \sum_{x_i, x_j \in \chi}   \frac{\eta}{d^2} C_{ij}(x_i, x_j) \right] +  \nonumber\\
    &\sum_{i \in \mathcal{V}} \left[ \log\left( \sum_{x \in \chi} \exp\left( -\eta C(x)     \right)            \right) +  \sum_{x \in \chi } \frac{\eta}{d} C_{i}(x) \right].\label{equation::upper_bound_2}
\end{align}

Notice that the RHS of the equation above is positive since: $\sum_{i=1}^\ell \exp( a_i ) \geq \frac{1}{\ell} \sum_{i=1}^\ell \exp(a_i ) \geq  \exp\left(\frac{\sum_{i=1}^\ell a_i }{\ell}   \right)$ for all $\ell \in \mathbb{N}$ and all $a_1, \cdots, a_\ell \in \mathbb{R}$.
Combining Equations (\ref{eq::upper_bound_1}) and (\ref{equation::upper_bound_2}) and the observation that $L(0,0) \leq L(\lambda^{(1)}, \xi^1)$ (by virtue of Lemma \ref{lemma: improvement}) we obtain the final result.
\end{proof}

In the case when all entries of $C$ are positive it may be the case that $S \gg \| \exp\left(  -\eta C    \right)\|_1$. 

\subsection{Complete Proof of Theorem~\ref{theorem:convergence}}

In this section, we will complete the proof of Theorem~\ref{theorem:convergence} by handling the case of EMP-cyclic.
We require two additional technical lemmas on the $l_1$ distance between updated variables. We will use $r(\cdot)$ and $c(\cdot)$ to denote row and column sums respectively of joint distribution matrices.

\begin{lemma}\label{lemma:tv worse}
Let $a,b\in \Sigma_d$ be two points in the simplex and let $p \in \mathbb{R}^d_+$ s.t. $\min(a_i, b_i) \leq p_i \leq \max(a_i, b_i)$ for all $1 \leq i \leq d$. Let $c \in \Sigma_d$ defined as $c = \frac{p}{\sum_i p_i}$. Then:
\begin{equation*}
    \max( \| a-c \|_1, \| b-c \|_1) \leq \| a-b \|_1
\end{equation*}
\end{lemma}

\begin{proof}
We only need to prove that $\| a-c \| \leq \|a-b\|_1$. 
From $\min(a_i, b_i) \leq p_i \leq max(a_i,b_i)$ we obtain:

\begin{equation*}
    | a_i - p_i | + |b_i - p_i| = |a_i - b_i|.
\end{equation*}
Let $t = \frac{1}{\sum_i p_i}$. The following relationships hold:
\begin{align*}
    \|a-c\|_1 &= \sum_i |a-t p_i| = \sum_i | a_i - p_i + (1-t)p_i | \\
    &\leq \sum_i | a_i - p_i | + \sum_i | (1-t)p_i|.
\end{align*}
Note that

\begin{equation*}
\sum_i | (1-t)p_i| = |1-t| \sum_i p_i = \frac{ |1-t|}{t} = |\frac{1}{t} - 1| = | \sum_i p_i - 1| ,   
\end{equation*}

and 
\begin{equation*}
    \sum_i|b_i - p_i| \geq |  \sum_i b_i - p_i | = |1-\sum_i p_i| = |\sum_i p_i - 1|.
\end{equation*}

Therefore,

\begin{equation*}
    \|a-c\|_1 \leq \sum_i |a_i - p_i| + \sum_i | b_i - p_i| = \sum_i |a_i-b_i| = \|a-b\|_1.
\end{equation*}

The result follows.

\end{proof}

Let $A  \in \Sigma_{d\times d}$ with elements $a_{ij}$ be a matrix representing joint distribution probabilities. For $p = \begin{bmatrix}p_1 & \ldots & p_d \end{bmatrix}^\top  \in \Sigma_d$, define
\begin{align*}
\tilde A = \frac{1}{z}\begin{bmatrix} a_{11} \sqrt{ \frac{ p_1}{r(A)_1} } & \cdots &  a_{1d} \sqrt{ \frac{ p_1}{r(A)_1} } \\ 
\vdots & \ddots & \vdots \\
a_{d1} \sqrt{ \frac{ p_d}{r(A)_d} } & \cdots  & a_{dd} \sqrt{ \frac{ p_d}{r(A)_d} } \end{bmatrix}
\end{align*}
where $z$ is a normalization term, such that the new probabilities matrix sums to one. The notation $r(A)_i$ denotes the $i$th element of row sum vector $r(A)$.

\begin{lemma}\label{lemma:tv worse joint}
	The following inequality holds on the difference between $A$ and $\tilde A$:
	\begin{align*}
	\| c(\tilde A) - c(A) \|_1 \leq \| r(\tilde A) - r(A) \|_1
	\end{align*}
\end{lemma}
\begin{proof}
	\begin{align*}
	\| c(\tilde A) - c(A) \|_1 & =  \sum_{j = 1}^d \left| \sum_{i = 1}^d \frac{a_{ij}}{z} \left(\sqrt{ \frac{ p_i}{r(A)_i} }  - z \right) \right| \\
	& \leq \sum_{i, j} \frac{a_{ij}}{z} \left| \sqrt{ \frac{ p_i}{r(A)_i} }  - z \right| \\
	& =  \sum_i \frac{r(A)_i}{z} \left| \sqrt{ \frac{ p_i}{r(A)_i} }  - z \right|  \\
	& = \sum_i  \left| \sqrt{ \frac{ r(A)_i p_i}{z} }  - r(A)_i \right| \\
	& = \| r(\widetilde A) - r(A)\|_1.
\end{align*}
\end{proof}

This proof of Theorem~\ref{theorem:convergence} relies heavily on the primal and dual variables at given times throughout the algorithm. As such, it is necessary to define precise notation for these temporal events. We note that there are two loops in the algorithm: an outer loop that controls the iterations and an inner one that loops over all edges in $\EE$. The outer loop's current iteration is given by $k \geq 0$, as defined and updated in Algorithm~\ref{alg1}. We denote the current step of the inner loop by $t$ where $1 \leq t \leq 4|\EE|$. This is due to the fact that there are four projections for each edge ($\X_{ij \rightarrow i}$, $\X_{ij, i}$, $\X_{ij \rightarrow j}$, and $\X_{ij,j}$) in one full iteration for $\LL_2$. Thus the algorithm alternates between enforcing consistency between an edge and vertex and normalizing the local distributions.

The value of $\Gamma$ at iteration $k$ and step $t$ within iteration $k$ is denoted by $\Gamma^{(k, t)}$. 
For example, at the very start of the algorithm, we are at iteration $k = 1$ and step $t = 1$ with initial value $\Gamma^{(1, 1)}$, which is equal to $\exp(-\eta C)$ with normalized vertex marginal and edge joint distributions.
The constraint set onto which a projection is made at $t$ in any iteration is denoted by $\X^{(t)}$. Note that we drop $k$ in the constraint set notation because the order in which the projections occur is always the same.

\begin{proof}[Proof of Theorem \ref{theorem:convergence}]

Let $k^*$ be the first iteration such that the termination condition in Algorithm~\ref{alg1} with respect to $\epsilon$ is met. For $k$ such that $1 \leq k \leq k^*$, there exists $ij \in \EE$ such that $\| r(\Gamma_{ij}^{(k,1)}) - \Gamma_i^{(k,1)}\|_1 \geq \epsilon$ or $\| c(\Gamma_{ij}^{(k,1)}) - \Gamma_j^{(k,1)}\|_1 \geq \epsilon$.

First consider the case where $\| c(\Gamma_{ij}^{(k,1)}) - \Gamma_j^{(k,1)}\|_1 \geq \epsilon$.
Let $t$ be chosen such that $\X^{(t)} = \X_{ij\rightarrow j}$. Note that $\Gamma_{j}^{(k, t)}$ can move within the $\epsilon$-ball of $c(\Gamma_{ij}^{(k, t)})$ between times $1$ and $t$ of the $k$th iteration due to earlier projections involving vertex $j$. However, $\Gamma_{ij}^{(k, t')} = \Gamma_{ij}^{(k, 1)}$ for all $t' \leq t - 2$ because it is only updated at step $t - 2$ where $\X^{(t - 2)} = \X_{ij\rightarrow i}$. Then, by repeatedly applying the triangle inequality, we have
\begin{align*}
        \epsilon & \leq \|c(\Gamma_{ij}^{(k,1)}) - \Gamma_j^{(k,1)}\|_1 \\
        & \leq \|c(\Gamma_{ij}^{(k, t - 2)}) - \Gamma_j^{(k,1)}\|_1  \\ 
         & \leq \|c(\Gamma_{ij}^{(k,t - 2)}) - \Gamma_j^{(k,t)}\|_1 + \sum_{t' \in \T_{j,r}^{(t)} \cup \T_{j,c}^{(t)}} \|\Gamma_j^{(k,t')} - \Gamma_j^{(k,t' + 2)}\|_1  \\
         & \leq \|c(\Gamma_{ij}^{(k,t)}) - \Gamma_j^{(k,t)}\|_1 + \|c(\Gamma_{ij}^{(k,t)}) - c(\Gamma_{ij}^{(k,t - 2)})\|_1 \\
         &\quad + \sum_{t' \in \T_{j,r}^{(t)} \cup \T_{j,c}^{(t)}} \|\Gamma_j^{(k,t')} - \Gamma_j^{(k,t' + 2)}\|_1,
    \end{align*}

 where $\T^{(t)}_{j, r}$ and $\T^{(t)}_{j, c}$ are sets of times before $t$ where a projection (for row and column consistency, respectively) caused $\Gamma_j$ to be updated:
    \begin{align*}
       	\T_{j,r}^{(t)} & \define \{ t' < t \ : \  \exists \ell \in N_r(i) \ \st \  \X^{(t')} = \X_{j \ell \rightarrow j} \} \\
       	\T_{j,c}^{(t)} & \define \{ t' < t \ : \  \exists \ell \in N_c(i) \ \st \  \X^{(t')} = \X_{\ell j \rightarrow j} \}.
    \end{align*}

    Therefore, $\Gamma_j^{(k, t' + 2)}$ is the result of enforcing consistency with another edge of $i$ and then normalizing $\Gamma_j$. Let $e_{t'}$ denote the edge (incident on $j$) onto which projections are occurring at step $t' \in \T_{j,r}^{(t)} \cup \T_{j,c}^{(t)}$. From Lemma~\ref{lemma:tv worse}, if $t' \in \T_{j,r}{(t)}$, then
    \begin{align*}
    \| \Gamma_{j}^{(k, t')} - \Gamma_j^{(k, t' + 2)}\|_1 & \leq \| \Gamma_{j}^{(k, t')} - r(\Gamma_{e_{t'}}^{(k, t')})\|_1.
    \end{align*}
    If $t' \in \T_{j, c}^{(t)}$, then
    \begin{align*}
    \| \Gamma_{j}^{(k, t')} - \Gamma_j^{(k, t' + 2)}\|_1 & \leq \| \Gamma_{j}^{(k, t')} - c(\Gamma_{e_{t'}}^{(k, t')})\|_1.
    \end{align*}
    Similarly, by combining Lemma~\ref{lemma:tv worse} and Lemma~\ref{lemma:tv worse joint}, we have
    \begin{align*}
    \|c(\Gamma_{ij}^{(k,t)}) - c(\Gamma_{ij}^{(k,t - 2)})\|_1 \leq \|r(\Gamma_{ij}^{(k,t)}) - r(\Gamma_{ij}^{(k,t - 2)})\|_1 \leq \|\Gamma_{i}^{(k,t - 2)} - r(\Gamma_{ij}^{(k,t - 2)})\|_1.
    \end{align*}
    Note that since the variables are normalized at every even step, they are individually valid probability distributions, and so the Hellinger inequality can be applied. For distributions, $p$ and $q$, the inequality states
    \begin{align*}
        \frac{1}{4}\| p - q \|_1^2 \leq 2 h^2(p, q).
    \end{align*}
    Therefore,
    \begin{align*}
        \frac{\epsilon^2}{4(\deg (\G) + 1)} & \leq 2h^2(c(\Gamma_{ij}^{(k,t)}), \Gamma_j^{(k,t)}) + 2h^2(r(\Gamma_{ij}^{(k,t - 2)}), \Gamma_i^{(k,t - 2)}) \\
         & \quad + \sum_{t' \in \T_{j, r}^{(t)}}  2h^2(r(\Gamma_{e_{t'}}^{(k,t')}), \Gamma_j^{(k,t')})
          + \sum_{t' \in \T_{j, c}^{(t)}} 2h^2(c(\Gamma_{e_{t'}}^{(k,t')}), \Gamma_j^{(k,t')}) \\
        & \leq L^{(k+1, 1)} -  L^{(k, 1)}.
    \end{align*}
    The last inequality follows from telescoping over all steps in iteration $k$ due to Lemma~\ref{lemma: improvement}. This proof was for the case when $\| c(\Gamma_{ij}^{(k,1)}) - \Gamma_j^{(k,1)}\|_1 \geq \epsilon$. For the case when $\| r(\Gamma_{ij}^{(k,1)}) - \Gamma_i^{(k,1)}\|_1 \geq \epsilon$, the procedure is identical except we may ignore the term $\|c(\Gamma_{ij}^{(k,t)}) - c(\Gamma_{ij}^{(k,t - 2)})\|_1$ since $\Gamma_{ij}$ is constant within iteration $k$ until the projection onto $X_{ij \rightarrow i}$. Thus, the improvement lower bound still holds.

Putting these results together with Lemma~\ref{lemma: max value}, we see that as long as a single constraint is violated above the $\epsilon$ threshold at the start of an iteration, it is possible to show that the value of $L$ increases by at least $\epsilon^2/4(\deg(\G) + 1)$ during the iteration. This implies that EMP-cyclic terminates in at most $\ceil{\frac{4\mathcal{S}_0(\deg(\G) +1)}{\epsilon^{2}}}$ iterations.

\end{proof}

\subsection{Proof of Theorem \ref{theorem:combined}}

We start by defining a version of $\LL_2$ with slack vectors. Let $\nu$ be a vector indexed in a similar way as $\Gamma$, where $\{ \nu_{ij}, \nu_{ji}\}_{ij \in \mathcal{E}}$. 
We define the slack $\nu$ as $\nu_{ij} = \Gamma_{ij} \1 - \Gamma_i$ and $\nu_{ji} = \Gamma_{ij}^\top \1 - \Gamma_j$. Then we define the slack polytope $\LL_2^\nu$ as
\begin{align} %\label{eq:local polytope}
	\LL_2^\nu \define \left\{ \Gamma \geq 0 \ : \ \begin{array}{lr}
	\Gamma_i \in \Sigma_d & \forall i \in \V  \\
	\Gamma_{ij}\1  = \Gamma_i + \nu_{ij}  & \forall ij \in \EE\\  
	\Gamma_{ij}^\top \1  = \Gamma_j + \nu_{ji}  & \forall ij \in \EE\\
	\1^\top \Gamma_{ij} \1 = 1 & \forall ij \in \EE \\ \end{array} 
	\right\}.
\end{align}
Notice that by definition the slack vectors $\nu$ satisfy that, for all $ij\in\EE$, $\nu_{ij}^\top \1 = \nu_{ji}^\top \1 =0$.
% We define the slack $\LL_2^\nu$ polytope as:
% \begin{align*}
% \ref{proj ij}\quad \X_{ij\rightarrow i} & = \{ \Gamma \ : \ \Gamma_{ij} \1  =  \Gamma_i + \nu_{ij} \} &  \quad
% \ref{proj i} \quad \X_{ij,i} & = \{  \Gamma \ : \ \Gamma_i^\top \1 = 1, \ \1^\top \Gamma_{ij} \1= 1 \} \\  
% \ref{proj ji} \quad \X_{ij \rightarrow j} & = \{  \Gamma \ : \ \Gamma_{ij}^\top  \1  = \Gamma_j + \nu_{ji} \} & \quad
% \ref{proj j}  \quad \X_{ij,j} & = \{  \Gamma \ : \ \Gamma_j^\top \1 = 1,\ \1^\top \Gamma_{ij} \1 = 1\}
% \end{align*} \footnote{Notice that by definition slack vectors $\nu$ satisfy that for all $ij\in\mathcal{V}$, $\nu_{ij}^\top \1 =0$}
The main difference between $\LL_2$ and $\LL_2^\nu$ lies in that the joints do not marginalize exactly to the vertex probabilities but do so up to a slack. 
Consider the entropy-regularized linear program corresponding to $\LL_2^\nu$:
\begin{align}\label{eq:reg_slack}
    \tag{Reg-slack}
    \min \quad \<C, \Gamma\> - \frac{1}{\eta}H(\Gamma) \quad \st \quad \Gamma \in \LL_2^\nu,
\end{align}
Introducing the exact same ensemble of dual variables $\lambda, \xi$ as in the Lyapunov function derivation, its dual function equals
{\small
\begin{align}\label{eq:lyap_slack}
\begin{aligned}
L^\nu(\lambda, \xi) & = - \sum_{ij \in \EE} \sum_{x_i, x_j \in \chi} \exp\left(- \eta C_{ij}(x_i, x_j) - \lambda_{ij}(x_i) - \lambda_{ji}(x_j) - \xi_{ij}\right)  \\
& \quad - \sum_{ij \in \EE} \sum_{x_i, x_j \in \chi}  \Big(  \lambda_{ij}(x_i)\nu_{ij}(x_i) +. \lambda_{ji}(x_j)\nu_{ji}(x_j)  \Big) \\
& \quad - \sum_{i \in \V} \sum_{x \in \chi} \exp\left( -\eta C_{i}(x) - \xi_i +  \sum_{j \in N_r(i)} \lambda_{ij}(x) + \sum_{j \in N_c(i)} \lambda_{ji}(x) \right) \\
& \quad  - \sum_{ij \in \EE} \xi_{ij} - \sum_{i \in \V} \xi_i + \sum_{ij \in \EE} \sum_{x_i, x_j \in \chi} \exp(-\eta C_{ij}(x_i, x_j)) + \sum_{i}\sum_{x \in \chi} \exp(-\eta C_{i}(x)).
\end{aligned}
\end{align}
}%
Furthermore, if $\lambda^*, \xi^*$ were a set of optimal dual variables, the optimal primal $\Gamma^*$ can be computed via
\begin{align}
    \Gamma_{ij}^*(x_i, x_j) &= \exp\left(-\eta C_{ij}(x_i,x_j) - \lambda_{ij}^*(x_i) - \lambda_{ji}^*(x_j) - \xi^*_{ij}     \right) \label{eq::dualfeas1} \\
    \Gamma_i^*(x_i) &= \exp\left( -\eta C_{i}(x_i) - \xi^*_i +  \sum_{j \in N_r(i)} \lambda^*_{ij}(x_i) + \sum_{j \in N_c(i)} \lambda^*_{ji}(x_i) \right).\label{eq::dualfeas2}
\end{align}
They satisfy the same formulae as the problem without slack variables.  Since dual optimality is equivalent to primal feasibility, whenever an iterate of EMP satisfies slack of $\nu$, its corresponding primal solution is optimal for (\ref{eq:reg_slack}). 

We start with a useful manipulation lemma:

\begin{lemma}\label{lemma::projection_l2_nu}
Let $\nu, \nu'$ be two slack vectors and let $\Gamma \in \LL_2^\nu$. Assume $\| v'    \|_\infty  \leq \frac{1}{2d}$. 
\begin{enumerate}
\item If for all $ij \in \mathcal{E}$ and $i \in \V$, $\Gamma_i + \nu'_{ij} \in \Sigma_{d}$, then there exists a vector $\Gamma' \in \LL_2^{\nu'}$ such that
\begin{equation}
    \| \Gamma - \Gamma' \|_1 \leq 2\| \nu - \nu'\|_1.
\end{equation}
\item If $\nu = 0$\footnote{We do not require that $\Gamma_i + \nu'_{ij} \in \Sigma_{d}$}, then there exists a vector $\Gamma' \in \LL_2^{\nu'}$ such that
\begin{equation}
    \| \Gamma - \Gamma' \|_1 \leq 6 d \deg(\G) \| \nu'\|_1 .
\end{equation}
\end{enumerate}

\end{lemma}

\begin{proof}
First we consider the case when for all $ij\in \mathcal{E}$, $\Gamma_i + \nu_{ij}'$ is a valid distribution (in other words, all its entries are in $[0,1]$ and its values sum to $1$). In this case, we can argue for the existence of $\Gamma'$ via the following:

Let $\Gamma'_i = \Gamma_i$ for all $i\in \mathcal{V}$. Let $ij \in \mathcal{E}$ and observe that $\Gamma_{ij} \1 = \Gamma_i + \nu_{ij}$ and $\Gamma_{ij}^\top \1 = \Gamma_j + \nu_{ji}$. We invoke Lemma 7 in \cite{altschuler2017near} to claim the existence of $\Gamma_{ij}'$ such that $\Gamma_{ij}' \1 = \Gamma_i + \nu_{ij}'$ and $(\Gamma_{ij}')^\top \1 = \Gamma_j + \nu_{ji}'$ and
\begin{align*}
    \| \Gamma_{ij} - \Gamma_{ij}'\|_1 &\leq 2\left( \| \Gamma_i + \nu_{ij}  - \Gamma_i - \nu'_{ij}   \|_1 + \| \Gamma_j + \nu_{ji}  - \Gamma_j - \nu'_{ji}  \|_1 \right) \\
    &=2\left( \|  \nu_{ij}  - \nu'_{ij}   \|_1 + \|   \nu_{ji}  -  \nu'_{ji}  \|_1 \right).
\end{align*}
Setting $\Gamma'$ to be the ensemble with values $\{\Gamma'_i\}_{i\in \mathcal{V}}$ and $\{\Gamma_{ij}'\}_{ij \in \mathcal{E}}$ the result follows. 

 Now we consider the case when there exist $ij \in \mathcal{E}$ such that $\Gamma_i + v_{ij}'$ does not lie in the probability simplex. In this case we will have to define $\Gamma'_i$ different from $\Gamma_i$. Consider some $i \in \V$. Let $N(i)$ be the set of neighbouring vertices to $i$ and we abuse notation slightly and use $\nu_{ij}$ for $j \in N(i)$ to denote the slack on $i$ as of the edge marginal shared by $i$ and $j$. We define $\Gamma_i'$ in the following way:
 \begin{enumerate}
     \item If $\Gamma_i + \nu'_{ij} \in \Sigma_{d}$ for all $j \in N(i)$ then let $\Gamma_i' = \Gamma_i $.
     \item Otherwise, let $\{ x_1, \cdots, x_r \} \subseteq [d]$ be the entries of $\Gamma_i$ such that for all $x_\tau \in \{ x_1, \cdots, x_r \} $ there exists at least one $j \in N(i)$ for which $\left[\Gamma_i + \nu_{ij}'  \right](x_\tau) \not\in [0,1]$. Therefore, we must define $\Gamma_i'$ such that 
     \begin{align*}
     \max_j\|\nu_{ij}'\|_\infty \leq \Gamma_i(x) \leq 1 - \max_j\|\nu_{ij}'\|_\infty,
     \end{align*}
     which can be done by taking the convex combination of $\mu_i$ with the uniform distribution:
     \begin{align*}
         \Gamma_i' = (1 - \theta) \Gamma_i + \frac{\theta}{d}.
     \end{align*}
     Setting $\theta = d  \max_j\|\nu_{ij}'\|_\infty$ guarantees this outcome because we are given that $\|\nu\|_\infty \leq \frac{1}{2d}$. Furthermore, we have
     \begin{align*}
         \| \Gamma_i - \Gamma_i'\|_1 & = \sum_{x} | \Gamma_i(x) - \Gamma_i'(x)| \\
         &  \leq 2 d \max_j\|\nu_{ij}'\|_\infty.
     \end{align*}
     This, in turn, implies $\sum_{i \in V} \| \Gamma_i - \Gamma_i'\|_1 \leq 2d \| \nu'\|_1$. Then, we apply the result of \cite{altschuler2017near} again to achieve existence of $\{\Gamma_{ij}'\}_{ij \in \EE}$ such that
     \begin{align*}
         \| \Gamma_{ij} - \Gamma_{ij}'\|_1  &\leq 2\left( \| \Gamma_i' - \Gamma_i - \nu'_{ij}   \|_1 + \| \Gamma_j'  - \Gamma_j - \nu'_{ji}  \|_1 \right) \\
    &=2\left( \|   \nu'_{ij}   \|_1 + \|   \nu'_{ji}  \|_1  + \| \Gamma_i - \Gamma_i'\|_1 + \| \Gamma_j - \Gamma_j'\|_1 \right).
     \end{align*}
     Summing over these yields $\sum_{ij \in \EE}  \| \Gamma_{ij} - \Gamma_{ij}'\|_1 \leq 2\|\nu'\|_1 + 2d\deg(\G) \|\nu'\|_1$. Therefore $\| \Gamma - \Gamma'\|_1 \leq 6d \deg(\G) \| \nu ' \|_1$

 \end{enumerate}

 \end{proof}

 We additionally require a similar lemma which allows us to project from one polytope to another while bounding the probabilities away from zero. 
 
 \begin{lemma}\label{lemma::projection_l2_nu_tau}
 Fix $\tau$ such that $0 < \tau \leq \frac{1}{8d^2}$ and a slack vector $\nu$ such that $\| \nu\|_\infty \leq \frac{1}{4d}$. If $\Gamma \in \LL^\nu_2$, then there exists a vector $\Gamma' \in \LL_2$ such that
 \begin{align*}
     \Gamma_i'(x_i) & \geq  \tau \quad \forall  i \in \V, \ x_i \in \chi \\
     \Gamma_{ij}'(x_i, x_j) & \geq \tau  \quad \forall ij \in \EE,\ x_i, x_j \in \chi \\
     \| \Gamma - \Gamma'\|_1 & \leq  2 \| \nu\|_1 + 2( m+ n) d^2 \tau.
 \end{align*}
 If $\Gamma \in \LL_2$, then there exists a vector $\Gamma' \in \LL_2^\nu$ such that
 \begin{align*}
     \Gamma_i'(x_i) & \geq  \tau \quad \forall  i \in \V, \ x_i \in \chi \\
     \Gamma_{ij}'(x_i, x_j) & \geq \tau  \quad \forall ij \in \EE,\ x_i, x_j \in \chi \\
     \| \Gamma - \Gamma'\|_1 & \leq  6 d \deg(\G) \| \nu\|_1 +  8 (|\EE| + n) d^2 \tau.
 \end{align*}
 \end{lemma}
 
 \begin{proof}
 We address each case individually.
 \begin{enumerate}
     \item We use the first result from Lemma~\ref{lemma::projection_l2_nu}, which yields $\widehat \Gamma \in \LL_2$ such that $\| \Gamma - \widehat \Gamma \|_1 \leq 2 \| \nu  \|_1$. If the probabilities are already bounded below $\tau$, then we are done; however, we must handle the worst case. As in the proof of Lemma~\ref{lemma::projection_l2_nu}, we compute a convex combination of $\widehat \Gamma$ with the uniform distribution to draw the distribution away from zero values. Define
     \begin{align*}
            \Gamma_{i}' & := (1 - \theta) \widehat \Gamma_i + \frac{\theta}{d} \1  \\
            \Gamma_{ij}' & := (1 - \theta) \widehat \Gamma_{ij} + \frac{\theta}{d^2} \1.
     \end{align*}
     where we set $\theta = \tau d^2$ which ensures that $\theta \in [0, 1]$ and $\Gamma \geq \tau$. Then, note that
     \begin{align*}
         \| \widehat \Gamma_i - \Gamma_i'\|_1 & = \sum_{x} | \frac{\theta}{d} - \theta \widehat\Gamma_{i}(x)| \leq 2\tau d^2 \\
         \| \widehat \Gamma_{ij} - \Gamma_{ij}'\|_1 & = \sum_{x_i, x_j} | \frac{\theta}{d} - \theta \widehat\Gamma_{ij}(x_i, x_j)| \leq 2\tau d^2.
     \end{align*}
     By the triangle inequality, we have
     \begin{align*}
         \| \Gamma - \Gamma'\|_1 & \leq \| \Gamma -\widehat \Gamma\|_1 + \| \widehat \Gamma - \Gamma'\|_1 \leq 2\|\nu\|_1 + 2 (n + m)  d^2 \tau.
     \end{align*}
     
    \item In the second case, we start by constructing a distribution $\delta$, which is nearly uniform but lives in the slack polytope is and bounded away from zero by at least $\tau$.

    For each $i \in \V$, we take $\delta_i \in \Sigma_d$ to be the uniform distribution where $\delta_i(x) = \frac{1}{d} \geq \tau$. Since $\|\nu \|_\infty \leq \frac{1}{4d}$, we perturb the uniform distribution with $\nu$ for each $j \in N(i)$, generating $\delta_{i}^j := \delta_i + \nu_{ij}$. Again, we are abusing notation slightly by using $\nu_{ij}$ to denote marginalization of edge $ij$ to vertex $i$. Note that $\delta_{i}^j \in \Sigma_d$, so we can define the product distribution $\delta_{ij} = \delta_i^j (\delta_j^i)\top \in \U_d(\delta_i^j, \delta_j^i)$, which, by construction, marginalizes such that the full vector $\delta$ given by the ensemble $\{ \delta_{i}\}_{i \in \V}$ and $\{ \delta_{ij} \}_{ij \in \EE}$ is in $\LL_2^{\nu}$. Furthermore, each component can be bounded below as \begin{align*}
    \delta_{ij}(x_i, x_j) & = \frac{1}{d^2} + \frac{\nu_{ij}(x_i)}{d} + \frac{\nu_{ji}(x_i)}{d} + \nu_{ij}(x_i) \nu_{ji}(x_j) \\
    & \geq \frac{1}{d^2} - \frac{1}{2d^2} - \frac{1}{16d^2} \\
    & \geq \frac{1}{4d^2}.
    \end{align*}
    Now, as before, we know there exists $\widehat \Gamma \in \LL_2^2$ such that $\|\Gamma -\widehat \Gamma\|_1 \leq 6 d\deg(\G) \| \nu\|_1$ from Lemma~\ref{lemma::projection_l2_nu}. Therefore, we can take the convex combination of $\Gamma' = (1 - \theta) \Gamma + \theta \delta$ to get $\Gamma' \in \LL_2^\nu$ such that $\Gamma' \geq \frac{\theta}{4d^2}$ in all entries.
    
    Taking $\theta = 4d^2 \tau \in [0, 1]$ ensures that $\Gamma' \geq \tau$. Furthermore, the difference can be computed as
    \begin{align*}
        \| \widehat \Gamma_i - \Gamma'_i \|_1 & = \sum_x | \frac{\theta}{d } - \theta \widehat \Gamma(x) | \leq 8 d^2 \tau \\
         \| \widehat \Gamma_{ij} - \Gamma_{ij}'\|_1 & = \sum_{x_i, x_j} | \theta \delta_{ij}(x_i, x_j) - \theta \widehat\Gamma_{ij}(x_i, x_j)| \leq 8d^2\tau. 
    \end{align*}
    Therefore, we have $\| \widehat \Gamma - \Gamma'\|_1 \leq 8 (|\EE| + n) d^2 \tau$, which by triangle inequality implies
    \begin{align*}
        \|  \Gamma - \Gamma' \|_1 \leq  \| \widehat \Gamma - \Gamma \|_1 +  \|  \widehat \Gamma - \Gamma' \|_1 \leq 6 d \deg(\G) \| \nu \|_1 + 8 (|\EE| + n) d^2 \tau.
    \end{align*}
    
 \end{enumerate}
 \end{proof}

We have now the necessary ingredients to prove the first theorem of this section, which provides a bound on the $l_1$ distance between the final iterate $\Gamma^{k}$ of Algorithms \ref{alg1} and \ref{alg2} and the solution $\Gamma_\eta^*$ of (\ref{eq:reg}). Crucially we analyze these iterates under the assumption all their component distributions $\Gamma^{(k)}_i$ for $i\in\mathcal{V}$ and $ \Gamma^{(k)}_{ij}$ for ${ij} \in \mathcal{E}$ are normalized. %\aldo{Add somewhere the note that we are assuming $\Gamma^{(k)}$ has all its distributions normalized}

\begin{theorem}\label{theorem:appendix_preliminary_combined}
Let $\Gamma^{(k)}$ is the $k$th iterate of EMP and let $\nu^{(k)}$ be the slack vector corresponding to $\Gamma^{(k)}$ such that $\| \nu^{(k)}\|_\infty \leq \frac{1}{4d}$. In other words,
\begin{align*}
    \nu^{(k)}_{ij} &= \Gamma^{(k)}_{ij} \1 - \Gamma^{(k)}_i \\
    \nu^{(k)}_{ji} &= \left( \Gamma^{(k)}_{ij}\right)^\top \1 - \Gamma^{(k)}_j.
\end{align*}
Fix $\tau > 0$ such that $\tau \leq \frac{1}{8d^2}$.
 Let $\Gamma^{(k)}(2)$ be the pseudo-marginal vector in $\mathbb{L}_2$ produced by the first case of Lemma \ref{lemma::projection_l2_nu_tau} when fed with  $\Gamma^{(k)}$ and $\tau$. Then,
%  \begin{align*}
%   \sum_{i \in \mathcal{V}} \frac{1}{2} \left\| \left(\Gamma^{(k)}(2)\right)_i - \left(\Gamma_\eta^*  \right)_i   \right\|_1^2 + \sum_{ij \in \mathcal{E}} \frac{1}{2} \left\| \left(\Gamma^{(k)}(2)\right)_{ij} - \left(\Gamma_\eta^* \right)_{ij}    \right\|_1^2 &\leq 2\frac{n \log(d ) + |\mathcal{E}|\log(d^2 )}{\eta} \\
%     &\quad +  6 \| C \|_\infty \|\nu^{(k)}  \|_1.
% \end{align*}

\begin{align*}
    \sum_{i \in \mathcal{V}} \frac{1}{2} \left\| \left(\Gamma^{(k)}(2)\right)_i - \left(\Gamma_\eta^*  \right)_i   \right\|_1^2 + \sum_{ij \in \mathcal{E}} \frac{1}{2} \left\| \left(\Gamma^{(k)}(2)\right)_{ij} - \left(\Gamma_\eta^* \right)_{ij}    \right\|_1^2 \\ \leq \left(\eta \|C\|_\infty + \log 1/\tau \right)\left( 8 d \deg(\G) \| \nu\|_1 +  10 (|\EE| + n) d^2 \tau\right).
\end{align*}

\end{theorem}
\begin{proof}
By definition $\Gamma^{(k)} \in \LL_2^{\nu^{(k)}}$. In fact, $\Gamma^{(k)}$ is the optimizer of the following regularized linear program:
\begin{align*}
    \min \quad \<C, \Gamma\> - \frac{1}{\eta}H(\Gamma) \quad \st \quad \Gamma \in \LL_2^{\nu^{(k)}},
\end{align*}
This observation follows because $\Gamma^{(k)}$ is in $\LL_2^{\nu^{(k)}}$ and its elements can be written as in  (\ref{eq::dualfeas1}) and (\ref{eq::dualfeas2}), thus satisfying dual feasibility.

Recall that after every iteration all the component distributions are normalized. 
 Recall that
\begin{align*}
    \langle \eta C, \Gamma^{(k)}(2) \rangle - H(\Gamma^{(k)}(2)) &= \D_\Phi\left(\Gamma^{(k)}(2), \exp({-\eta C})\right) + \< \1, e^{-\eta C}\> \\
    \langle \eta C, \Gamma_\eta^* \rangle - H(\Gamma_\eta^*) &= \D_\Phi\left(\Gamma_\eta^*, \exp({-\eta C})\right) + \< \1, e^{-\eta C}\>,
\end{align*}
% \lee{should have $\frac{1}{\eta}$ on right hand side + constant $\< \1, e^{-\eta C}\>$}.
where $\Phi = -H$ is the negative entropy. The point $\Gamma_\eta^*$ is the optimal point of the information projection $\exp\left(-\eta C\right)$ for points in $\LL_2$. By the properties of information projections,
\begin{equation*}
    \D_\Phi\left(\Gamma^{(k)}(2), \exp({-\eta C})\right) \geq  \D_\Phi\left(\Gamma^{(k)}(2),\Gamma_\eta^*\right) +  \D_\Phi\left(\Gamma_\eta^*, \exp({-\eta C})\right).
\end{equation*}
Since for $\Gamma^{(k)}(2)$ and $\Gamma_\eta^*$, the sum of their entries is the same, by Pinsker's inequality (applied to each of the component vertex and edge distributions) this in turn implies that
\begin{align}
%   \underbrace{ 
  \D_\Phi\left(\Gamma^{(k)}(2), \exp({-\eta C})\right) - \D_\Phi\left(\Gamma_\eta^*, \exp({-\eta C})\right)  
%   }_{I} 
  &\geq  \D_\Phi\left(\Gamma^{(k)}(2),\Gamma_\eta^*\right) \nonumber \\
  &\geq \sum_{i \in \mathcal{V}} \frac{1}{2} \left\| \left(\Gamma^{(k)}(2)\right)_i - \left(\Gamma_\eta^*  \right)_i   \right\|_1^2 +\\
  &\quad \sum_{ij \in \mathcal{E}} \frac{1}{2} \left\| \left(\Gamma^{(k)}(2)\right)_{ij} - \left(\Gamma_\eta^* \right)_{ij}    \right\|_1^2 . \label{eq::divergence_kl_pinsker}
\end{align}
Let $\Gamma_\eta^*(\nu^{(k)})$ in $\LL_2^{\nu^{(k)}}$ be the vector produced by Lemma~\ref{lemma::projection_l2_nu_tau} applied to $\Gamma_{\eta}^* \in \LL_2$. Note that we utilize the existence of $\Gamma_\eta^*(\nu^{(k)})$ and $\Gamma^{(k)}(2)$ for analysis but we need not actually \textit{compute} them. Expanding $I$ yields
\begin{align*}
    \D_\Phi\left(\Gamma^{(k)}(2), \exp({-\eta C})\right) -  \D_\Phi\left(\Gamma_\eta^*, \exp({-\eta C})\right)  &= \langle \eta C, \Gamma^{(k)}(2)  - \Gamma_\eta^* \rangle + H(\Gamma_\eta^*) - H(\Gamma^{(k)}(2)) \\
    &=
    \underbrace{
    \langle \eta  C, \Gamma^{(k)}(2)  - \Gamma^{(k)} \rangle + H(\Gamma^{(k)}) - H(\Gamma^{(k)}(2))
    }_{A_1}
    \\
    &\quad  + 
    \underbrace{
    \langle  \eta C, \Gamma^{(k)}  - \Gamma_\eta^*(\nu^{(k)}) \rangle + H(\Gamma_\eta^*(\nu^{(k)})) - H(\Gamma^{(k)})
    }_{A_2}
    \\
    &\quad + 
    \underbrace{
    \langle \eta C, \Gamma_\eta^*(\nu^{(k)})  - \Gamma_\eta^* \rangle + H(\Gamma_\eta^*) - H(\Gamma_\eta^*(\nu^{(k)}))
    }_{A_3}
    .
\end{align*}
Term $A_2$ is negative since $\Gamma^{(k)}$ is the optimal point in the slack polytope. Because $\Gamma_\eta^*(\nu^{(k)})$ and $\Gamma^{(k)}(2)$ were constructed such that all their probabilities are lower bounded by $\tau$, it holds that the entropies are $\log \frac{1}{\tau}$-Lipschitz in $\|\cdot\|_1$
Terms $A_1$ and $A_3$ can be then bounded:
\begin{align*}
    A_1 & \leq \eta \|C \|_\infty  \| \Gamma^{(k)}(2)  - \Gamma^{(k)} \|_1  + \log\frac{1}{\tau} \| \Gamma^{(k)}(2)  - \Gamma^{(k)}\|_1 \\
    & \leq  ( \eta \|C\|_\infty + \log 1/\tau) \left( 2 \| \nu^{(k)}\|_1 + 2(|\EE| + n) d^2 \tau \right) \\
    A_3 & \leq  \|C \|_\infty  \| \Gamma_\eta^*(\nu^{(k)})  - \Gamma_\eta^* \|_1 + \log\frac{1}{\tau} \| \Gamma_\eta^*(\nu^{(k)})  - \Gamma_\eta^*\|_1 \\
    & \leq ( \eta \|C\|_\infty + \log 1/\tau) \left( 6 d \deg(\G) \| \nu^{(k)}\|_1 +  8 (|\EE| + n) d^2 \tau\right).
\end{align*}
The result then follow as
\begin{align*}
    A_1 + A_3 
    & \leq \left(\eta \|C\|_\infty + \log 1/\tau \right)\left( 8 d \deg(\G) \| \nu^{(k)}\|_1 +  10 (|\EE| + n) d^2 \tau\right).
\end{align*}

\end{proof}
Theorem \ref{theorem:appendix_preliminary_combined}, combined with the EMP algorithm's optimality condition can provide convergence guarantees for the case when $\LL_2$ is tight and the solution is unique. We restate the main result, Theorem \ref{theorem:combined}, for readability.

\begin{theorem}
Let $\eta \geq \frac{2\log(16  n^2d^2 ) + 16|\mathcal{E}|d^2  }{\min(\Delta, \frac{1}{128})} $, and $\epsilon^{-1} > (25d\deg(\G) |\EE|)^2 \max \left( \eta \|C\|_\infty,68 \right)$. If $\LL_2$ is tight and $|\mathcal{V}^*_2| = 1$, the EMP  algorithm returns a MAP assignment after 
$\ceil{\frac{4\mathcal{S}_0(\deg(\G) +1)}{\epsilon^{2}}}$ iterations for EMP-cyclic and after $\ceil{\frac{4 \mathcal S_0}{\epsilon^2}}$ iterations for EMP-greedy.
\end{theorem}

\begin{proof}
Let $\Gamma^{(k)}$ be the last internal iterate of the EMP algorithm before rounding. Since the stopping condition has been met, the slack vector $\nu^{(k)}$ corresponding to $\Gamma^{(k)}$ must satisfy $\| \left(\nu^{(k)}  \right)_{ij} \|_1 \leq \epsilon$ for all $ij\in \mathcal{E}$ so that $\| \nu^{(k)}\|_1 \leq 2 |\EE| \epsilon$. 

Let $\Gamma^{(k)}(2)$ be defined as in Theorem  \ref{theorem:appendix_preliminary_combined} and choose $\tau = \frac{\epsilon}{10(  | \EE| + n) d^2 }$\footnote{As long as $\epsilon \leq \frac{1}{4d}$ at least, this guarantees $\tau \leq \frac{1}{8d^2}$, so we are free to use Theorem~\ref{theorem:appendix_preliminary_combined}}. Then, the bound from Theorem~\ref{theorem:appendix_preliminary_combined} becomes
\begin{align*}
    & \left(\eta \|C\|_\infty + \log 1/\tau \right)\left( 8 d \deg(\G) \| \nu^{(k)}\|_1 +  10 (|\EE| + n) d^2 \tau\right) \\
    & \leq \left(\eta \|C\|_\infty + \log 1/\tau \right)\left( 16 d \deg(\G) |\EE| \epsilon  +  10 (|\EE| + n) d^2 \tau\right) \\
    & = \left(\eta \|C\|_\infty + \log \frac{10(  | \EE| + n) d^2} { \epsilon} \right)17 d \deg(\G) |\EE| \epsilon \\
    & = \left(\eta \|C\|_\infty + \log \left(   10(| \EE| + n)  d^2  \right) + \log \frac{1}{\epsilon} \right)17 d \deg(\G) |\EE| \epsilon \\
    & \leq \left(\eta \|C\|_\infty + \log \left(   10(| \EE| + n)  d^2  \right) + 2\epsilon^{-1/2} \right)17 d \deg(\G) |\EE| \epsilon,
\end{align*}
where the last inequality used the fact that $\log(x) \leq n (x^{1/n} - 1)$ for $n > 0$. Choosing $\epsilon^{-1} > 425d^2 \deg(\G)^2 |\EE|^2 \max \left\{ \eta \|C\|_\infty,68 \right\} $ ensures that
\begin{align*}
    \sum_{i \in \mathcal{V}} \frac{1}{2} \left\| \left(\Gamma^{(k)}(2)\right)_i - \left(\Gamma_\eta^*  \right)_i   \right\|_1^2 + \sum_{ij \in \mathcal{E}} \frac{1}{2} \left\| \left(\Gamma^{(k)}(2)\right)_{ij} - \left(\Gamma_\eta^* \right)_{ij}    \right\|_1^2 & \leq \frac{3}{25}.
\end{align*}

% \begin{align*}
%   \sum_{i \in \mathcal{V}} \frac{1}{2} \left\| \left(\Gamma^{(k)}(2)\right)_i - \left(\Gamma_\eta^*  \right)_i   \right\|_1^2 + \sum_{ij \in \mathcal{E}} \frac{1}{2} \left\| \left(\Gamma^{(k)}(2)\right)_{ij} - \left(\Gamma_\eta^* \right)_{ij}    \right\|_1^2 &\leq 2\frac{n \log(d ) + 2|\mathcal{E}|}{\eta} + \\
%     &\quad 6 \| C \|_\infty \|\nu^{(k)}  \|_1\\
%     &\leq 2\min(\Delta, \frac{1}{128}) + 6\epsilon |\mathcal{E}| \| C \|_\infty\\
%     &\leq \frac{1}{64} + \frac{1}{16}\\
%     &=\frac{5}{64}.
% \end{align*}
Consequently for all $i\in \mathcal{V}$
\begin{equation*}
    \left\|  \left(\Gamma^{(k)}(2)\right)_i -\left( \Gamma_\eta^* \right)_i\right \|_1 \leq \frac{2}{5}.
\end{equation*}
and for all $ij \in \mathcal{E}$
\begin{equation*}
    \left\|  \left(\Gamma^{(k)}(2)\right)_{ij} - \left(\Gamma_\eta^*\right)_{ij} \right \|_1 \leq \frac{2}{5}.
\end{equation*}

We also have
\begin{align*}
    \| \Gamma^{(k)}(2) - \Gamma^{(k)}\|_1 & \leq 2  \| \nu^{(k)}\|_1 + 2(|\EE| + n) d^2 \tau  \\
    & \leq 4 |\EE| \epsilon + \frac{\epsilon}{5}  \\
    & \leq 5 |\EE| \epsilon,
\end{align*}
which implies $\| \Gamma^{(k)}(2) - \Gamma^{(k)}\|_1  \leq \frac{1}{24}$ and $\| \Gamma_\eta^* - \Gamma^*  \|_1 \leq \frac{1}{32}$ (by the condition on $\eta$, see Theorem \ref{prop:approximation error}). Putting these inequalities together by triangle inequality,
\begin{align*}
    \| \left(\Gamma^{(k)}\right)_i - \left(\Gamma^*\right)_i \|_1 &\leq   \| \left(\Gamma^{(k)}\right)_i  - \left(\Gamma^{(k)}(2)\right)_i   \|_1 +  \| \left(\Gamma^{(k)}(2)\right)_i  - \left(\Gamma^*_\eta\right)_i   \|_1 + \|  \left(\Gamma^*_\eta\right)_i -\left(\Gamma^*\right)_i     \|_1 \\
    &\leq  \frac{1}{24} +\frac{2}{5} + \frac{1}{32} \\
    &< \frac{1}{2}.
\end{align*}
For all $i \in \mathcal{V}$. A similar statement holds for all $ij \in \mathcal{E}$:
\begin{align*}
    \| \left(\Gamma^{(k)}\right)_{ij} - \left(\Gamma^*\right)_{ij} \|_1 &\leq   \| \left(\Gamma^{(k)}\right)_{ij}  - \left(\Gamma^{(k)}(2)\right)_{ij}   \|_1 +  \| \left(\Gamma^{(k)}(2)\right)_{ij}  - \left(\Gamma^*_\eta\right)_{ij}   \|_1 +     \|  \left(\Gamma^*_\eta\right)_{ij} -\left(\Gamma^*\right)_{ij}     \|_1 \\
    &\leq \frac{1}{24} + \frac{2}{5} + \frac{1}{32} \\
    &< \frac{1}{2}.
\end{align*}
Therefore, assuming $\Gamma^*$ (the solution of $\LL_2$) is integral,
\begin{equation*}
   \left( \mathrm{round}(\Gamma^{(k)}) \right)_i = \left(\Gamma^* \right)_i \text{ for all } i \in \mathcal{V}
\end{equation*}
and
\begin{equation*}
   \left( \mathrm{round}(\Gamma^{(k)}) \right)_{ij} = \left(\Gamma^* \right)_{ij} \text{ for all } ij \in \mathcal{E}.
\end{equation*}

\end{proof}

\section{Experiment Details}\label{section::experiment-details}

In this section, we provide some additional details for the experiments in Section~\ref{section::experiments}. As mentioned, empirical comparisons between state-of-the-art solvers and EMP-like algorithms have been studied extensively \citep{meshi2012convergence,ravikumar2010message,werner2007linear,kappes2013comparative}. For instance, \citet{meshi2012convergence} found that the regularized star-based message passing algorithms greatly outperform standard optimization techniques such as FISTA and gradient descent, which do not exploit the coordinate structure of the problem. 

The primary purpose of these experiments is to understand how the theoretical results in Section~\ref{section::convergence_analysis} manifest in a practical setting. In particular, we would like to understand how the convergence rates, in terms of the ability to round to the solution, behave as a function of the parameters of the problem such as graph size, choice of regularization $\eta$, and connectivity of the graph. In all experiments, we ran an LP solver on the graph in order to obtain the ground-truth MAP assignment. We only considered problems that were tight. The solver specifically is the ECOS solver through a CVXPY wrapper.

\subsection{Grid Experiments}
As mentioned, our first set of experiments considered solving the MAP problem on $\sqrt n \times \sqrt n$ grids, totallying $n$ vertices. The vertices were connected by edges to their vertical and horizontal neighbors in the grid. This setting is fairly standard in the literature \citep{erdogdu2017inference,globerson2008fixing,ravikumar2010message}.

We considered the MAP problem with $d = 3$ labels and choose a cost vector $C$ in the family of multi-label Potts models, another well-studied application \citep{wainwright2008graphical}. Potts models typically have diagonal potentials between edges. That is, we only penalize/reward when the labels on two connected vertices agree. We randomly generated the actual values of the vector. For vertex costs, we chose $C_i(x_i) \sim \text{Unif}(-0.5, 0.5)$ and for the edge costs we chose
\begin{align*}
    C_{ij}(x_i, x_j) & = \begin{cases}
        \beta_{ij} & x_i = x_j\\
        0 & \text{otherwise}
    \end{cases} \quad \forall ij, x_i, x_j,
\end{align*}
where $\beta_{ij}\sim\text{Unif}\{-0.1, 0.1\}$.
In the approximation results, we ran the algorithms until they had effectively converged after 80 iterations, where each iteration consisted of a full pass over the edges. For EMP-cyclic, this means simply going through all the edges once. For EMP-greedy, one iteration means the opportunity to update each edge exactly once, (although the algorithm will greedily select them in reality). Thus both algorithms update the same number of edges, though their choices will be different. Regardless, we found that 80 iterations was reasonably sufficient to observe the approximation properties. We measured the results in terms of the average Hamming distance between the LP's solution, which is integral, and the rounded solution returned by the algorithms.

\subsection{Random Graph Experiments}

While the grid topology offers a consistent platform to evaluate the algorithms, we also considered randomly generated graphs, specifically Erd\H{o}s-R\'enyi random graphs. These graphs are constructed by iterating through every pair of the $n$ vertices. Then, an edge is drawn between vertex $i$ and $j$ with probability $p$. Specifically, we chose $p = \frac{1.1 \log n}{n}$, which is just large enough that the graph is almost surely connected.
We found these to be useful hyperparameter because any lower and the graph would largely be disconnected. Any higher and typically we found the LP was not tight. We chose the same multi-label Potts model for generating the cost vector $C$.

With these experiments, we intended to understand how diverse graph topologies would affect convergence due to randomness. In particular, we restricted the degrees of the graph to $\deg(\G) = 5, 10$ to observe how the algorithms behave on denser graphs.

\end{document}